%% file: main.tex
\documentclass[11pt]{article}

\usepackage{natbib}
\usepackage{fullpage}
\usepackage[utf8]{inputenc} % allow utf-8 input
\usepackage[T1]{fontenc}    % use 8-bit T1 fonts
\usepackage{hyperref}       % hyperlinks
\usepackage{url}            % simple URL typesetting
\usepackage{booktabs}       % professional-quality tables
\usepackage{amsfonts}       % blackboard math symbols
\usepackage{nicefrac}       % compact symbols for 1/2, etc.
\usepackage{microtype}      % microtypography
\usepackage{comment}
\usepackage{graphicx}
\usepackage{subfigure}
\usepackage{booktabs} % for professional tables
\usepackage{mathtools}
\usepackage{xcolor}
\usepackage{algorithm}
\usepackage{tikz}
\usepackage{float}
\usepackage{amssymb}
\usepackage{amsmath}
\usepackage{amsthm}
\usepackage{algorithmic}

\newtheorem{theorem}{Theorem}
\newtheorem{lemma}{Lemma}
\newtheorem{corollary}{Corollary}
\newtheorem{definition}{Definition}

\title{Stealthy Adversarial Attacks on Stochastic Multi-Armed Bandits}

\author{
    Zhiwei Wang \\
    Tsinghua University\\
    \texttt{zhiweithu@gmail.com} \\
    \and
    Huazheng Wang\\
    Oregon State University\\
    \texttt{huazheng.wang@oregonstate.edu}
    \and
    Hongning Wang\\
    Tsinghua University\\
    \texttt{wang.hongn@gmail.com}\\
}
\date{}
\begin{document}

\maketitle

\begin{abstract}
Adversarial attacks against stochastic multi-armed bandit (MAB) algorithms have been extensively studied in the literature. In this work, we focus on reward poisoning attacks and find most existing attacks can be easily detected by our proposed detection method based on the test of homogeneity, due to their aggressive nature in reward manipulations. This motivates us to study the notion of stealthy attack against stochastic MABs and investigate the resulting attackability. Our analysis shows that against two popularly employed MAB algorithms, UCB1 and $\epsilon$-greedy, the success of a stealthy attack depends on the environmental conditions and the realized reward of the arm pulled in the first round. We also analyze the situation for general MAB algorithms equipped with our attack detection method and find that it is possible to have a stealthy attack that almost always succeeds. This brings new insights into the security risks of MAB algorithms.
\end{abstract}

\section{Introduction}
\label{submission}

In a stochastic multi-armed bandit (MAB) problem, a learner each time takes an arm from a presented set to interact with the environment for reward feedback, where the reward is assumed to be i.i.d. sampled from an unknown but fixed distribution \cite{auer2002finite,Auer02,agrawal2017near,lattimore2020bandit}. The learner's goal is to maximize its cumulative rewards in a finite number of interactions. 
As such algorithms continuously learn from external feedback, their adversarial robustness has attracted increasing attention in the community \cite{liu2019data,ma2018data,jun2018adversarial}. The most well-studied adversarial setting is the reward poisoning attack, where an attacker can selectively modify the reward of the learner's pulled arms to deceive the learner. Accordingly, the attacker can have two distinct goals: with a high probability, 1) force the learner to take a particular arm a linear number of times, i.e., known as a targeted attack; or 2) make the learner suffer from linear regret, i.e., known as an untargeted attack, both subject to sublinear attack cost constraint. It is known that most of the popular stochastic MABs algorithms can be easily manipulated \citep{jun2018adversarial,liu2019data}, indicating a serious security vulnerability for the practical use of MAB algorithms. 

The concerns have therefore spurred great research efforts in developing adversarially robust stochastic MAB algorithms.
The existing efforts are mainly focused on developing new algorithms with provable robustness guarantees \citep{lykouris2018stochastic,gupta2019better,guan2020robust,liu2020action}. Various robust reward estimators or exploration strategies are introduced to tolerate the adversarial reward corruption, which, however, are at the cost of increased regret.  
Our study shows that most existing reward poisoning attacks can be effectively detected via the test of homogeneity \cite{buishand1982some}. 
As a result, most existing MAB algorithms can be easily protected by such a detection method, with little impact on their regret. 
This provides a new perspective to examine the robustness or the so-called attackability \cite{wang2022linear} of MAB algorithms under the presence of attack detection methods.    

This paper focuses on the targeted attack setting against stochastic MAB algorithms. In Section \ref{section 4}, we first introduce a method based on the test of homogeneity to detect possible reward poisoning attacks. Our key insight is that existing attack methods \citep{jun2018adversarial,liu2019data} aggressively push the realized rewards of non-target arms below that of the target arm, such that the observed reward sequence on non-target arms is no longer i.i.d. samples from the same distribution. 
As a result, the test of homogeneity is in a good position to actively detect such data poisoning attacks. We demonstrate that this method exhibits a low type-I error, suggesting its reliability. We then prove that with MAB algorithms like UCB1 or $\epsilon$-greedy, if an attack method can succeed with a high probability, this attack can always be detected with a non-negligible probability. Consequently, our detection method also demonstrates a low type-II error. These results establish two key findings: 1) the proposed detection method is highly effective, and 2) existing attack methods can be easily detected using our approach.

In light of the limitations posed by existing attack methods in the presence of our proposed detection method, we introduce the concept of stealthy attack in bandit problems in Section \ref{stealthy attacks}. The results of Section \ref{section 4} also show that when the learner applies UCB1 or $\epsilon$-greedy algorithms, no attack method can simultaneously achieve both stealthiness and efficiency, under the conditions specified by the reward gap, i.e., the bandit instance’s attackability. We then propose a stealthy attack algorithm that works when the bandit instance is attackable. We should note that this limitation
stems from the detection against reward poisoning attacks,
rather than a flaw in the attack design. Also, it is important to note here that our work is not concerned with how the learner should react after detecting an attack, but mainly with the question of attackability of a MAB algorithm under the presence of the proposed detection mechanism, and accordingly how an attack can be carried out.

We then analyze the possibility of a stealthy attack against two most popularly used MAB algorithms, UCB1, and $\epsilon$-greedy \cite{auer2002finite}, equipped with our proposed attack detection method. 
We point out that except for the cases where we prove no stealthy attack can succeed, the stealthy attack we propose can be successful against both algorithms. 
Next, we analyze the feasibility of stealthy attacks against general MAB algorithms. 
We use construction to show that there are algorithms and corresponding attack methods where stealthy attacks can almost always succeed. 
This suggests that for the most general MAB algorithms, the limitations on the feasibility of the stealthy attack previously demonstrated for UCB1 and $\epsilon$-greedy do not hold. 
But when we restrict the randomness of the algorithm itself, we find that we can prove a parallel result. This also opens up a new direction of research in adversarial bandit algorithms. 

In addition to the theoretical analysis, we also performed an extensive set of experiments based on simulations to validate our results. We demonstrated that existing attacks against UCB1 and $\epsilon$-greedy can be easily detected by the test of homogeneity. And the feasibility of a stealthy attack depends on the environment (i.e., the ground-truth mean of the target arm) and the realized reward of the first pulled arm, which is out of anyone's control.  

To summarize, the main contributions of this work are,
\begin{itemize}
    \item We propose a simple but effective method to detect reward poisoning attacks against stochastic MAB algorithms. We show that the detection method is highly effective and existing attack methods can be detected using our approach. This leads us to study a stealthy attack against MAB algorithms. We show that stealthy attacks can only be successful under certain circumstances; in other words, the attackability of MAB algorithms is not universal, and there is a trade-off between stealthiness and effectiveness.
    \item We propose a stealthy attack method that can successfully manipulate the UCB1 or $\epsilon$-greedy algorithm protected by the detection method, except the cases proved to be not attackable in nature. Then, we construct two examples to show that for general MAB algorithms, especially when the algorithm itself is stochastic, stealthy attacks can almost always succeed. But for algorithms whose randomness in arm selection is limited, the success of stealthy attacks still depends on the environmental condition and the reward of the first pulled arm.
\end{itemize}

\section{Related Work}

Due to the wide adoption of bandit algorithms in practice, increasing amount of research attention has been spent on adversarial attacks against such algorithms to understand their robustness implications. To date, most effort has been focused on data poisoning attacks against stochastic MAB ~\citep{jun2018adversarial, liu2019data}, combinatorial MAB~\citep{balasubramanian2023adversarial} and linear contextual bandit~\citep{garcelon2020adversarial,wang2022linear} algorithms. And these methods follow the same principle: deliberately lower the reward of non-target arms, to deceive the learner to pull the target arm a linear number of times. 
Our research shows that if the bandit algorithm is protected by the test of homogeneity-based attack detection, most existing attack method will fail because the range in which the reward can be lowered is limited. Other work focusing on attacks against linear contextual bandits takes into account the situation where attackers can modify historical data~\cite {ma2018data}. Some recent works studied action poisoning attacks~\cite {liu2020action} and adversarial attacks on online learning to rank~\cite {wang2023adversarial}. Recently, there are also studies on reward poisoning attacks against reinforcement learning~\citep{behzadan2017vulnerability,huang2019deceptive,ma2019policy,sun2020vulnerability,zhang2020adaptive,liu2021provably}. However, none of these existing studies consider the possible existence of attack detection, which nullifies the attack's real-world efficacy/implication.

On the defense side, there is also an increasing amount of research to improve the robustness of bandit algorithms against adversarial attacks. \citet{lykouris2018stochastic} introduced a multi-layer active arm elimination method to improve the bandit algorithm's robustness against reward poisoning attacks. The key idea is to use increased confidence intervals to tolerate reward corruptions. But the resulting regret also degrades linearly to the amount of corruption. \citet{gupta2019better} extended the solution by performing phased arm elimination with overlapping arm sets to avoid eliminating a good arm too early. This reduces the cost of regret for being robust. 
\citet{guan2020robust} proposed to use a median-based estimator together with calibrated pure exploration for robust bandit learning. \citet{feng2020intrinsic} proved that general MAB algorithms such as Thompson Sampling, UCB, and $\epsilon$-greedy are robust when the attacker is not allowed to decrease the realized reward of each pulled arm. \citet{liu2020action} shifted the mean reward estimation by the difference between the estimated confidence intervals between the best and worst arms to improve their bandit algorithm's robustness. \citet{bogunovic2021stochastic} proposed two algorithms with near-optimal regret for the stochastic linear bandit problem, under known and unknown attack budget respectively. And the key insight is to expand the confidence interval for exploration. \citet{ding2022robust} applied the same principle to develop contextual linear bandit algorithms robust to both reward  and context poisoning attacks.  

The notion of attackability was first studied in \cite{wang2022linear} under linear stochastic bandits, where the authors suggested that the geometry spanned by the target arm, the optimal arm, and the ground-truth bandit parameter vector decides the attackability of the learning problem. As a result, some linear stochastic bandit instances are naturally robust. Our work shares similar a spirit from this angle, but we focus on the MAB setting, which was believed to be always attackable in prior work.   

\section{Preliminaries}

We study reward poisoning attacks against stochastic multi-armed bandit algorithms~\cite{auer2002finite}. Basically, a bandit game consists of $N$ arms with unknown but fixed $\sigma$-sub-Gaussian reward distributions $\{F_{1},\cdots, F_{N}\}$ centered at $\{\mu_{1}, \ldots, \mu_{N}\}$. %The learner don't know $\left\{\mu_{i}\right\}$. 
%We assume that the whole game lasts for $T$ rounds. 
The length of the time horizon is $T$ and predetermined.
At each round $t$, the learner pulls an arm $I_{t} \in [N]$ and receives reward $r_{t}^{0}$ from the environment following $F_{I_{t}}$. %The reward $r_{t}^{0}$ generated by the environment is drawn from $F_{I_{t}}$. 
The performance of the bandit algorithm is measured by its pseudo-regret, which is defined as $R_T = \mathbb{E}[\sum_{t=1}^{T}(\mu_{i^{*}}-\mu_{I_t}))]$, where $i^*$ is the best arm at hindsight, i.e., ${i^*} = \arg \max_{i\in[N]} \mu_i $. The learner's goal is to minimize $R_{T}$. 

The attacker sits in-between the environment and the learner. At each round $t$ after the learner chooses to play arm $I_{t}$, the attacker manipulates the reward into $r_{t}=r_{t}^{0}-\alpha_{t}$, which is then presented to the learner. If the attacker decides not to attack, $\alpha_{t}=0$. We call $\alpha_{t} \in \mathbb{R}$ the attack manipulation. Without loss of generality, assume arm $K$ is the target arm, which is not the optimal arm in hindsight: $\mu_{K}<\max _{i=1 \ldots N} \mu_{i}$. Define the cumulative attack cost as $C(T)=\sum_{t=1}^{T}|\alpha_{t}|$. The attacker's goal is to force the learner to choose the target arm a linear number of times with a sublinear attack cost. Or formally, we consider the attack successful if after $T$ rounds the attacker spends $o(T)$ cumulative attack cost and forces the learner to choose the target arm for $T-o(T)$ times.

\section{Detection Method}\label{section 4}
We first introduce our proposed detection method against adversarial reward poisoning attacks. The key idea is that since the attacker manipulates the rewards (i.e., lower the reward of non-target arms), the rewards observed by the learner on the same arm are no longer iid samples from the same distribution. Therefore, the learner can use the test of homogeneity \cite{buishand1982some} to detect possible manipulations of observed rewards so far. Once the attack is detected, the learner can resort to different means to handle the adversarial situation, e.g., restart the reward estimation. But the design of those different approaches is out of the scope of this paper.

Define the history of pulls before time $t$ as $\mathcal{H}_{t} =$ $\left\{(I_{s}, r_{I_{s}},c_{s})\right\}_{s=1}^{t}$, which represents all information at time $t$. Let $N_{i}(t)$ be the number of observations associated with arm $i$ up to time $t$. Define $\hat{\mu}_{i}^{0}(t):=$ $N_{i}(t)^{-1} \sum_{\left\{s: s \le t, I_{s}=i\right\}} r_{s}^{0}$ as the empirical pre-attack mean reward of arm $i$ up to time $t$, and $\hat{\mu}_{i}(t):=$ $N_{i}(t)^{-1} \sum_{\left\{s: s \le t, I_{s}=i\right\}} r_{s}$ as the corresponding empirical post-attack mean reward. Define function $\beta(n,\delta)$ and event set $E_1$ as follows,
\begin{align*}
    \beta(n,\delta):=&\sqrt{\frac{2 \sigma^{2}}{n} \log \frac{\pi^{2} N n^{2}}{3 \,\delta}},\\
    E_{\delta}:=&\left\{\forall i, \forall t \geq N:\left|\hat{\mu}_{i}^{0}(t)-\mu_{i}\right|<\beta\left(N_{i}(t), \delta\right)\right\},
\end{align*}
where $\delta \in (0,1)$. Notice that $\beta(n,\delta)$ monotonically decreases with $n$. The following lemma shows that the distance between the pre-attack and ground-truth mean rewards of all arms is bounded by $\beta(N_{i}(t),\delta)$ with high probability.  

\begin{lemma}\label{lemma:detection basis}
$\text { For } \delta \in(0,1), \mathbb{P}(E_{\delta})>1-\delta$.
 \end{lemma}

Lemma \ref{lemma:detection basis} suggests a method for the learner to detect if the observed reward sequence has been manipulated, and the idea is simple: if the learner finds a set of empirical means for an arm that is too widely distributed, then $E_{\delta}$ does not hold, which rarely happens in the absence of attacks.

\textbf{Detection Method.} At the beginning of a bandit game, the learner chooses a parameter $\delta \in (0,1)$. The learner runs the following hypothesis test for $\forall t \in [T]$ until the null hypothesis is rejected. At each time $t$, the null hypothesis (i.e., no attack so far) and alternative hypothesis (i.e., there is an attack) are as follows:
\begin{equation*}
    \begin{aligned}
        && &H_{0}^{t} : \text{the learner has not been attacked at and before }t 
        \\
        &&\ &H_{1}^{t} : \text{the learner has been attacked at or before }t 
    \end{aligned}
\end{equation*}
The learner rejects the null hypothesis $H_{0}^{t}$, if $\exists i \in [N] $ such that
\begin{equation}
\label{eq:hyp-test}
   \bigcap_{j\in[t], N_{i}(j)>0}\Big(\hat{\mu}_{i}(j)-\beta(N_{i}(j),\delta),\hat{\mu}_{i}(j)+\beta(N_{i}(j),\delta)\Big) = \emptyset. 
\end{equation}

The following lemma shows one aspect of the effectiveness of our detection method: the proposed detection method has a low type-I error. 
\begin{lemma}\label{lemma: detection error}  The probability of the union of all the type-I errors introduced by Eq.\eqref{eq:hyp-test} can be upper bounded by $\delta$. In our problem setting, the type-I error refers to the detection method's erroneous claim of an attack when there is, in fact, no attack present.
\end{lemma}

In the remainder of this paper, we shall assume that the learner adopts the detection method corresponding to a fixed parameter $\delta$ in conjunction to its chosen bandit algorithm, and we abbreviate $\beta(n,\delta)$ as $\beta(n)$.  
We study the setting where $\delta$ is publicly known and in practice, the attacker should treat $\delta$ as a hyper-parameter for fine tuning.

In the remaining part of Section \ref{section 4}, we show another aspect of the effectiveness of our detection method. Lemma \ref{lemma: UCB1 unattackable}-Corollary \ref{corollary: epsilon unattackable} demonstrates the proposed detection method exhibits a high probability of detecting effective attacks (i.e., a low type-II error). Specifically, If the detection method is applied on top of popularly employed MAB algorithms, such as UCB1 or $\epsilon$-greedy, we can show effective reward poisoning attacks against  MABs \cite{jun2018adversarial,liu2019data} can be successfully detected with a decent probability.

We first consider the case where the learner applies UCB1 \cite{auer2002finite}, which selects an arm $I_{t}$ to play at time $t$ using the following rule:
\begin{equation*}
I_{t}= \begin{cases}t, & \text { if } t \leq N \\ \arg \max _{i}\left\{\hat{\mu}_{i}(t-1)+3 \sigma \sqrt{\frac{\log t}{N_{i}(t-1)}}\right\},  & \text { otherwise }\end{cases}
\end{equation*}
For $1 \leq i,j \leq N$, define $\Delta^{0}_{ij} = \hat{\mu}_{i}(N)-\mu_{j}$, $\Delta_{ij} = \mu_{i}-\mu_{j}$. The following lemma gives the upper bound of target arm (i.e., arm $K$) pulls related to the realized reward of the first pulled arm $r_1 = \hat{\mu}_{1}(N)$, before the detection method declares the attack. This lemma points out the fundamental reason why effective attacks are always easily detectable.
\begin{lemma}\label{lemma: UCB1 unattackable}
Suppose the learner runs the UCB1 algorithm to choose arms. For any history of pulls $\mathcal{H}_{T}$ where the attack is not detected, if $\Delta^{0}_{1K}>\beta(1)$,  with probability at least $1-\frac{1}{T^{3}}$ the number of rounds that target arm $K$ is pulled up to time $T$: $N_{K}(T)$ is bounded by
$$
\max \left\{\frac{3 C(T)}{\Delta^{0}_{1K}-\beta(1)}, \frac{81 \sigma^{2} \log T}{(\Delta^{0}_{1K}-\beta(1))^{2}},\left(\frac{\pi^2N}{3\delta}\right)^{\frac{2}{5}}\right\}
$$
where $C(T)$ is the cumulative attack cost.
\end{lemma}

The above lemma shows that under the condition $\Delta^{0}_{1K}>\beta(1)$, when the detection method fails to detect the attack, with a high probability the attack itself will not succeed: either the target arm will not be pulled linear times or the attack cost cannot be sublinear. 

We explain the dependency on the first pulled arm's realized reward $\widehat{\mu}_{1}(N)$. Due to the attack is not detected, the total amount of changes an attacker can make on an arm's average reward is limited. When the realized reward of the first pulled arm is very large, the attacker will not be able to reduce the average reward of this arm below the target arm, and thus fails to be effective (i.e., the target arm will be pulled $T-o(T)$ time). Otherwise, this attack will be detected. We will show in Section 5.1 that when $\widehat{\mu}_{1}(N)$ is small, there exists an attack method that will not be detected and has a high probability of success.

Based on Lemma \ref{lemma: UCB1 unattackable}, we have the following corollary regarding the efficiency of attack detection:
\begin{corollary}\label{corollary: UCB1 unattackable}
Suppose the learner runs the UCB1 algorithm. The attacker applies any attack methods such that it can fool the learner to pull the target arm $T-o(T)$ times at a cost of $o(T)$ with a high probability for any $T$ large enough, i.e., with probability at least $1-\epsilon$ and for $T$ is larger than a constant $T_{0}$. Given any $\Delta_{0}>0$, there exists a constant $T_{1}$ related to $\Delta_{0}$ such that when $T>T_{1}$, the attack will be detected with probability at least $1-\frac{\epsilon}{(1-\frac{1}{T^{3}})}-\mathbb{P}(\widehat{\mu}_{1}(N) \leq \mu_K + \beta(1) + \Delta_{0}).$
\end{corollary}

This corollary suggests that for any effective attack against the UCB1 algorithm, the detection will succeed with at least a certain probability related to the environment when $T$ is large enough. In particular, this conclusion is also true for existing attacks against MABs \cite{jun2018adversarial,liu2019data}. 

We can explain the result as follows: when $\Delta^{0}_{1K}>\beta(1) + \Delta_{0}$, by Lemma \ref{lemma: UCB1 unattackable} we know that the attack with failed detection will hardly be effective. Hence, for the attack method with a high probability of success, the detection only fails with a very low probability: no more than $\frac{\epsilon}{(1-\frac{1}{T^3})}$. And when $\Delta^{0}_{1K} \leq \beta(1) + \Delta_{0}$, it is clear that the detection fails with a probability no more than $\mathbb{P}(\widehat{\mu}_{1}(N) \leq \mu_K + \beta(1) + \Delta_{0})$. Combining these two results gives an upper bound on the probability of detection failure.

We also note that the probability lower bound in the corollary depends on the environment conditions, and it is non-negligible when the ground-truth mean $\mu_{1}$ is much larger than $\mu_{K}$. This points out the effectiveness of our detection method. By using the properties of $\sigma$-sub-Gaussian distribution, it is easy to prove that when $\Delta_{1K} > \Delta_{0} + \beta(1) + \sqrt{2\sigma^2\log\frac{1}{p}}$, the probability lower bound in the Corollary \ref{corollary: UCB1 unattackable} is larger than $1 - \frac{\epsilon}{1-\frac{1}{T^3}} - p \approx 1 - p$.
We emphasize that the ground-truth mean in the environment is unbounded, therefore the difference between the ground-truth means(i.e., $\Delta_{1K}$) can be great. 

Next, let us consider the case that the learner employs the $\epsilon$-greedy algorithm. The learner plays each arm once for $t=1,\cdots, N$. For $t>N$
$$
I_{t}= \begin{cases}\operatorname{draw~uniform}[N], & \text { w.p. } \epsilon_{t} \quad \text { (exploration) } \\ \arg \max_{i} \hat{\mu}_{i}(t-1), & \text { otherwise (exploitation) }\end{cases}
$$ 
We can still prove the result that is parallel to the case when the learner applies UCB1, despite the randomness of $\epsilon$-greedy.
\begin{lemma}\label{lemma: epsilon unattackable}
Suppose the learner runs the $\epsilon$-greedy algorithm with $\epsilon_{t}=\min\{1, \frac{CN}{t}\}$ and $C \geq 3$. Given any $\eta \in (0,1)$, for any interaction
sequence $\mathcal{H}_{T}$ where the attack is not detected, if $\Delta^{0}_{1K}>\beta(1)$, with probability at least $1-\eta-\frac{1}{T^{3.5}}-\frac{1}{\log T}$ the number of rounds that target arm $K$ is pulled up to time $T$ can be bounded as follows,
\begin{align*}
&N_{K}(T) \leq C\left(1 + \log\frac{T}{CN}\right) + \sqrt{3(C + C\log\frac{T}{CN})\log \frac{1}{\eta}} \\&+ \max \left\{C_1\log T, \frac{81 \sigma^{2} \log T}{(\Delta^{0}_{1K}-\beta(1))^{2}}, \frac{3 C(T)}{\Delta^{0}_{1K}-\beta(1)}, C_2\right\}
\end{align*}
where $C_1 = e^{\frac{5a}{9}-1}C^{2}N^{2}$, $C(T)$ is the total attack cost, $C_2 = CNe^{\frac{a}{6c}-\frac{1}{2}}$, $a = \beta^{-1}\left(\frac{\Delta^{0}_{1K}-\beta(1)}{3}\right)$.
\end{lemma}
Similar to the previous claim, the above lemma points to the main reason for the effectiveness of the detection method. It also states that the success of an attack under the previous detection method depends on the realized reward of the first pulled arm. As in the case of $\epsilon$-greedy, we further explicitly point out the effectiveness of the detection method by the following corollary.

\begin{corollary}\label{corollary: epsilon unattackable}
Suppose the learner runs $\epsilon$-greedy algorithm with $\epsilon_{t}=\min\{1, \frac{CN}{t}\}$ and $C \geq 3$ to choose arm. The attacker applies any of the attack methods such that it can fool the learner to pull the target arm $T-o(T)$ times at a cost of $o(T)$ with a high probability for any $T$ large enough, i.e., with probability at least $1-\epsilon$ and for $T$ is larger than a constant $T_{0}$. Given $\Delta_{0}>0$ and $\eta \in (0,1)$, there exists a constant $T_{1}$ related to $\Delta_{0}$ and $\eta$ such that when $T>T_{1}$ the attack detection will be successful with probability at least $1-\frac{\epsilon}{(1-\eta-{T^{3.5}}-\frac{1}{\log T})}-\mathbb{P}(\widehat{\mu}_{1}(N) \leq \mu_K + \beta(1) + \Delta_{0}).$
\end{corollary}

\section{Stealthy Attacks}\label{stealthy attacks}
In this section, we propose the concept of a stealthy attack and discuss its feasibility when the learner applies different learning algorithms. Under the restriction of stealthy attacks, the magnitude of each attack is limited, and we demonstrate that the attack on the algorithm is no longer necessarily feasible, but rather depends on the environment and the realized reward. 

\subsection{Stealthy Attacks against UCB1 and $\epsilon$-greedy}\label{subsection: 5.1}

We first provide the definition of a stealthy attack, which basically means that an attack method will hardly be detected with our detection. We then propose a stealthy attack method against UCB1 and $\epsilon$-greedy, and we prove its effectiveness under some conditions related to the environment and reward realization.  

\textbf{Stealthy Adversarial Attack.} Assume that the learner's algorithm, as well as the parameter $\delta$ of the detection method, have been determined, and the learner runs the detection method as described in Section \ref{section 4}. We say that an attack algorithm is \textit{stealthy} if for any given environment $\{F_1,\cdots, F_N\}$ we have that the detection of the attack throughout the game can succeed with probability at most $\delta$.
In other words, we require the attack to be non-detectable.

Note that we require the detection of the stealthy attack be successful with probability at most $\delta$, which is because the probability of type-I error of our detection mechanism is at most $\delta$. We believe that such an attack is stealthily enough. 

\textbf{The Limitation of Stealthy Attack.} From Lemma \ref{lemma: UCB1 unattackable}, \ref{lemma: epsilon unattackable} we know that when $\Delta^{0}_{1K}>\beta(1)$ and the detection of the attack against UCB1 or $\epsilon$-greedy failed, with high probability the attack will also failed. Combined with the definition of a stealthy attack, it is easy to demonstrate that for any stealthy attack, when $\Delta^{0}_{1K}>\beta(1)$ holds, then with high probability it will not succeed. This shows that under certain conditions, stealthy attacks cannot succeed in essence. The reason is that the attacker must limit the strength of attacking the non-target arms to avoid detection, which in turn leads to the attack failure under some conditions.

\textbf{Stealthy attack against UCB1 and $\epsilon$-greedy.} Next we show that for the remaining cases, i.e. when $\Delta^{0}_{1K} < \beta(1)$, there exists a stealthy attack algorithm against both UCB1 and $\epsilon$-greedy that can succeed with high probability.

Now we give the attack method. Suppose $\eta \in (0,1)$ is chosen by the attacker. The attacker attacks arms in the following way: 

For the first $N$ rounds($N$ is the number of arms) the attacker attacks in the following way: when $t<N$, for arm $i>1$ and $i\neq K$ the attacker attacks the arm and spends minimal attack cost to make $\hat{\mu}_{i}(N) \leq \hat{\mu}_{1}(N) - 2\beta(1,\eta) -2\beta(1) -d$, $d \geq 0$ is a constant chosen by the attacker. And when the attacker attacks different learning algorithms, $d$ can be adjusted accordingly.
 
 After the first $N$ rounds, the attack happens when the learner chooses to play an arm $I_{t} \neq K$. If $I_{t} \neq 1$ the attacker gives an attack $\alpha_{t} = \alpha_{i}$ where $\alpha_{i}$ is the attack cost the attacker used to attack arm $i$ for the first time. If $I_{t} = 1$ the attacker computes an attack $\alpha_{t}$ with the value such that 
$$
\hat{\mu}_{1}(t) = \hat{\mu}_{1}(N) - \beta(1) - \beta(N_{1}(t-1)+1).
$$

We now show that this attack method is stealthy and will succeed with high probability when $\Delta^{0}_{1K} < \beta(1)$.
\begin{lemma}\label{lemma: attack method stealthy}
The attack method stated above is stealthy.
\end{lemma}
\begin{proof}
    For attacks against arm $1$ it's obvious that the attack detection will fail. For arm $i>1$ and $i\neq K$, $\forall 1 \leq t \leq T$, $\hat{\mu}_{i}(t) = \hat{\mu}^{0}_{i}(t) - \alpha_{i}$. Then we have $\hat{\mu}_{i}(t)-\hat{\mu}_{i}(j) = \hat{\mu}^{0}_{i}(t)-\hat{\mu}^{0}_{i}(j)$, by Lemma \ref{lemma: detection error} the detection to our attack will be successful with probability at most $\delta$. Hence, the attack is stealthy. 
\end{proof}

\begin{theorem}\label{theorem: UCB1 attack}
Suppose the learner applies the UCB1 algorithm. If $\Delta^{0}_{1K} < \beta(1)$, with probability at least $1 - \eta - \frac{2(\Delta^{0}_{1K}-\beta(1))^{7}}{7(9\sigma^{2}\log T)^{\frac{7}{2}}}$, the attacker forces the learner to choose the target arm in at least  $T-\left(\frac{9\sigma^{2}}{(\Delta^{0}_{1K}-\beta(1))^2}+\frac{9N\sigma^{2}}{d^2}\right)\log T -(N-1)$   rounds, and incurs a cumulative attack cost at most
\begin{align*}
&(\frac{18(\beta(1)+\beta(1, \eta))\sigma^{2}}{(\Delta^{0}_{1K}-\beta(1))^{2}} + \frac{9N\sigma^{2}(2\beta(1)+4\beta(1,\eta)+d)}{d^2})\log T \\&+ \left(\frac{9\sigma^{2}\log T}{d^2}+1\right)\sum_{i \neq 1,K}|\Delta_{1i}| + dN + 4\beta(1,\eta)N+4\beta(1)N
\end{align*}

\end{theorem}

Denote $\beta(1) - \Delta^{0}_{1K}$ as $\Delta$. The number of non-target arm pulls is
$O\left((\frac{\sigma^2}{\Delta^2} + \frac{N\sigma^2}{d^2})\log T + N\right)$ and the attack cost is 
$O\left(\left(\frac{\sigma^2}{\Delta^2} + \frac{N\sigma^2}{d^2}\right)\log T + \frac{\sigma^{2}\log T}{d^2}\sum_{i \neq 1,K}|\Delta_{1i}| + dN \right)$. We can see that a larger $d$ decreases the non-target arm pulls, and the attacker only needs to choose $d \leq \sqrt{N}\Delta$, because if $d > \sqrt{N}\Delta$ we have $\frac{\sigma^2}{\Delta^2} > \frac{N\sigma^2}{d^2}$. When choosing $d = \Theta(\Delta)$, the cost is $O(\frac{\sigma^2}{\Delta^2}(N + \sum_{i \neq 1,K}|\Delta_{1i}|)\log T + \Delta N)$. The probability is $1-\eta-O\left(\left(\frac{\Delta}{\sigma\sqrt{\log T}}\right)^7\right)$, when $T \to \infty$ it approaches $1-\eta$.

\begin{theorem}\label{theorem: epsilon attack}
Suppose the learner applies the same $\epsilon$-greedy algorithm as in Lemma \ref{lemma: epsilon unattackable}. If  $\Delta^{0}_{1K} < \beta(1)$, with probability at least $1 - 3\eta$, the attacker forces the learner to choose the target arm in at least 
$$ T - t_{2} - N\left(C + C\log\frac{T}{CN} + \sqrt{3(C + C\log\frac{T}{CN})\log \frac{1}{\eta}}\right)$$rounds, and using a cumulative attack cost at most 
\begin{align*}
    & C(T) \leq \left(C + C\log\frac{T}{CN} + \sqrt{3(C + C\log\frac{T}{CN})\log \frac{N}{\eta}}\right) (4N\beta(1, \eta) + 
    2N\beta(1) + \sum_{i \neq 1,K}|\Delta_{1i}| + dN) + \\& 2\beta(1)t_{2} + \beta(1, \eta))t_{2}
\end{align*}
where $t_2 = 4(1 + \frac{\sigma^2}{\Delta^2})(\frac{3\sigma^2}{4\Delta^2}\log \frac{4}{\eta} + \frac{2\sigma^2}{\Delta^2}\log\frac{4T}{\eta}) + 2\log\frac{4}{\eta} + t_1$, $t_1 = \max\{CNe^{\frac{5b}{6c}-\frac{3}{2C}\log\frac{\eta}{4} - \frac{1}{2}}, 5CN\}$, $b = \frac{2\sigma^2}{\Delta^2}\log (1+\frac{\Delta^2}{\sigma^2})$
\end{theorem}

We note that $b<2 \Rightarrow t_{1} = O(N), t_2 = O(\frac{\sigma^4}{\Delta^4}\log T + N)$. Notice the cost and non-target arm pulls, the attacker only needs to choose $d = 0$ and then the cost is $O((N + \sum_{i \neq 1, K}|\Delta_{1i}| + \frac{\sigma^4}{\Delta^4})\log T)$. Compared to the cost in Theorem \ref{theorem: UCB1 attack}, the $\log T$ term has a higher order of $4$ for $\frac{\sigma}{\Delta}$ and is thus more sensitive to its change.  

In addition, the reason that the feasibility of a stealthy and effective attack does not depend on the first realized rewards of the other non-target arms is that the attacker can use the realized reward of the first pulled arm to attack the other non-target arms by making their first observed reward very low. This is exactly what our algorithm does in this section. %Therefore, ultimately only the first reward of the first pulled arm plays the crucial role.

\subsection{Stealthy Attacks on General Algorithms}

Now let us consider a more general situation. We study the feasibility of the stealthy attack when the learning algorithms are not limited to either UCB1 or $\epsilon$-greedy. The following lemma shows that there exists an effective learning algorithm that corresponds to the existence of a stealthy attack algorithm that always succeeds in the attack. 

\begin{lemma}\label{lemma: construction}
We can find a learning algorithm such that given any sub-Gaussian environment $(F_{1},\cdots, F_{N})$, the number of arm pulls of any sub-optimal arm $i$ follows $E(N_{i}(T)) = \Theta(\log T)$, and simultaneously the attacker can find a corresponding stealthy attack method s.t. for any target arm $K \in \{1,\cdots, N\}$ with probability at least $1-\eta$ the attacker can force the learner to choose the target arm for $T-O(\log T)$ times with a cumulative attack cost at most $O(\log T \sqrt{\log (\frac{1}{\eta})})$.
\end{lemma}

The construction method is shown in the appendix. The design of our learning algorithm incorporates specific "action takings", or in other words, special ways of arm selection, which are then utilized by the attack method.

Lemma 5.4 demonstrates that it is possible to find attack methods that are always stealthy and effective to some bandit algorithms; but when the bandit algorithms have limited randomness, e.g., UCB1 and $\epsilon$-greedy, from Theorem \ref{lemma: UCB1 unattackable} and \ref{lemma: epsilon unattackable} we know that it might be impossible to find such an attack method. These two different results show that there are essential differences between different bandit algorithms in terms of the feasibility of a stealthy attack.

However, when we control the randomness of the distribution of rewards and the randomness of the algorithm itself, we can get results parallel to the previous cases with UCB1 and $\epsilon$-greedy. We call an algorithm $\it{effective}$ $\it{under}$ $ \it{the}$ $\it{control}$ $\it{of}$ $\it{the} $ $\it{reward}$ $\it{randomness}$ (ERR) (or $\it{effective}$ $\it{under}$ $ \it{the}$ $\it{control}$ $\it{of}$ $\it{the} $ $\it{algorithm}$ $\it{and}$ $\it{reward}$ $\it{randomness}$ (EARR)) if and only if it has the following property:
given any $\eta \in (0,1)$ and environment $(F_{1}, \cdots, F_{N})$. When a learner applies this algorithm, there exist two functions $g(t,\eta): \frac{g(t,\eta)}{t} \to 0  \text{ as } t \to  \infty$ and $l(n,\eta):l(n,\eta) \geq \beta(n)$ such that if $\forall i, \forall N \leq t \leq T$ the empirical mean of a history of pulls $\mathcal{H}_{T}$: $\widehat{\mu}_{i}(N_{i}(T))$ was bounded in the interval $ (\mu_{i}-l\left(N_{i}(t),\eta\right),\mu_{i}+l\left(N_{i}(t),\eta\right))
$, then for $T$ is large enough with probability at least $1-\eta$, the regret $R_{T}$ is bounded by $g(T,\eta)$: $R_{T} \leq g(T,\eta)$ (or then as long as $T$ is large enough the regret $R_{T}$ is bounded by $g(T)$: $R_{T} \leq g(T,\eta)$ under EARR). It is not hard to see that the common algorithms like UCB1, Thompson Sampling, and $\epsilon$-greedy are ERR, and UCB1 is an EARR algorithm.

For the EARR algorithm, We can prove the result parallel to the result in the previous subsection.
\begin{theorem}\label{theorem: EARR}
Suppose the learner runs an EARR algorithm to choose arms. Given any $\eta \in (0,1)$ and environment $(F_{1}, \cdots, F_{N})$, we can find a function $g(t, \eta)$: $\frac{g(t,\eta)}{t} \to 0  \text{ as } t \to  \infty$, such that for any history of pulls $\mathcal{H}_{t}$ where the attack is not detected, if $\Delta_{1K}^{0}>2\beta(1)$, with probability at least $1-\eta-\frac{1}{T^{8(\frac{1}{2}-\nu)^2}}$ that in $\mathcal{H}_{t}$ the number of rounds that target arm $K$ is pulled up to time $T$ is no more than
\begin{equation*}
\max \{ \dfrac{g(T,\eta)}{\nu (\Delta_{1K}^{0}-2\beta(1))},\frac{16 \sigma^{2} \ln T}{(\Delta_{1K}^{0}-2\beta(1))^{2}},\frac{2 C(T)}{\Delta_{1K}^{0}-2\beta(1)}\}
\end{equation*}
where $C(T)$ is the cumulative attack cost.
\end{theorem}
This theorem shows that the stealthy attack on an EARR algorithm will fail in some cases. But for a general ERR(which may not be EARR) algorithm we do not have the same result. We can prove a similar result to Lemma \ref{lemma: construction}:
\begin{lemma}\label{lemma: ERR}
We can find an ERR algorithm such that the attacker can find a corresponding stealthy attack method s.t. given any environment $(F_{1},\cdots, F_{N})$ for any target arm $K \in \{1,\cdots, N\}$ with probability at least $1-\eta$ the attacker can force the learner to choose the target arm for $T-O(\log T)$ times with a cumulative attack cost at most $O\left(\log T \sqrt{\log (\frac{1}{\eta})}\right)$.
\end{lemma}
The construction is similar to that in Lemma \ref{lemma: construction}. Combined with Theorem \ref{theorem: EARR}, this lemma suggests that there is a restriction on the possibility of a successful stealthy attack for EARR, but this does not directly hold for ERR.

\begin{figure*}[ht]
    \centering
    \begin{tabular}{c c}
     \includegraphics[width=6cm]{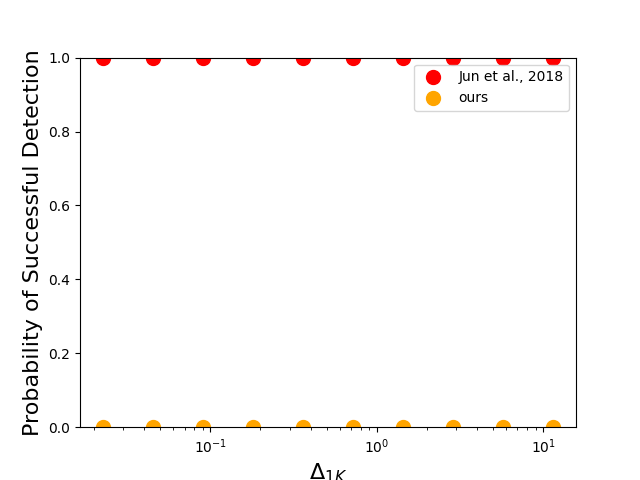} & \includegraphics[width=6cm]{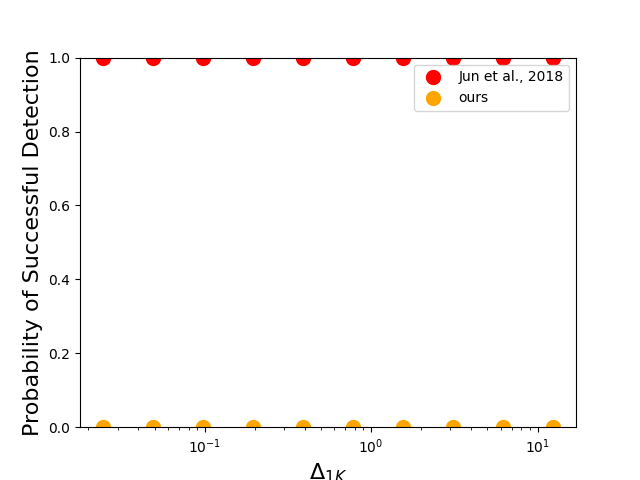}
    \end{tabular}
    \caption{Probability of successful detection under \cite{jun2018adversarial}'s attack method when UCB1 is the victim algorithm. Left: $(N,T) = (10, 10000)$. Right: $(N,T) = (30, 20000)$}
    \label{fig:p-detect-UCB1}
\end{figure*}

\begin{figure*}[ht]
    \centering
    \begin{tabular}{c c}
     \includegraphics[width=6cm]{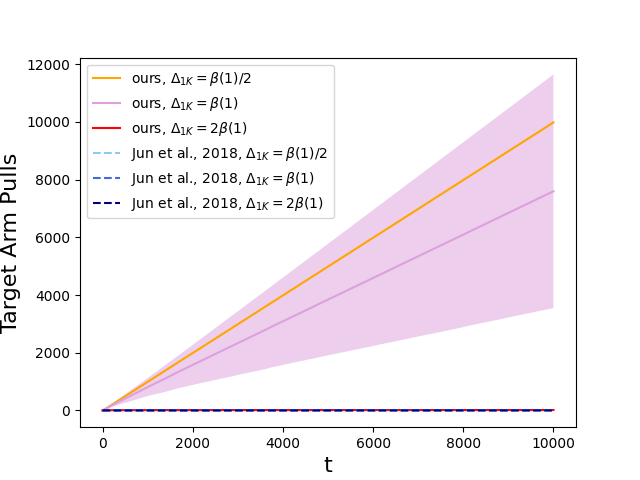} &
     \includegraphics[width=6cm]{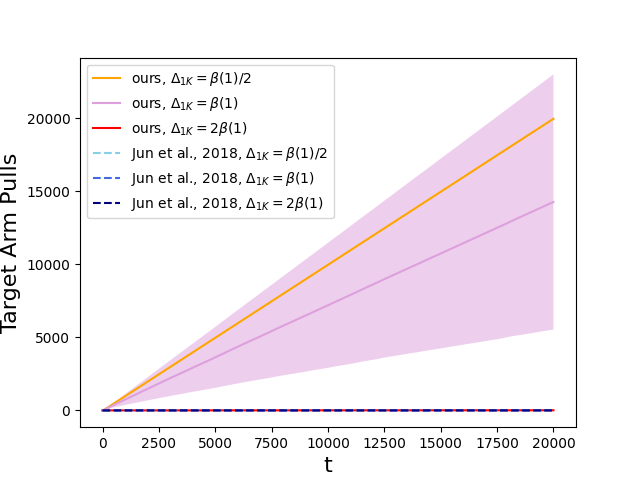} 
    \end{tabular}
    \caption{Target arm pulls under different attack methods when UCB1 is the victim algorithm. Left: $(N,T) = (10, 10000)$. Right: $(N,T) = (30, 20000)$}
    \label{fig:armpull-UCB1}
\end{figure*} 

\section{Experiments}\label{sec:exp}

We performed extensive empirical evaluations using simulation to verify our theoretical results against different MAB algorithms, attack methods, and environment configurations. We mainly present results for UCB1 here, specific results for $\epsilon$-greedy will be provided in the appendix.

\subsection{Experiment Setup} 

In our simulations, we execute the reward poisoning attack method proposed in \cite{jun2018adversarial} as our baseline and our attack algorithms against UCB1 and $\epsilon$-greedy algorithms in the presence of attack detection proposed in Section \ref{section 4}. We varied the number of arms $N$ in $\{10, 30\}$ and set each arm's reward distribution to an independent Gaussian distribution. The ground-truth mean reward $\mu_i$ of each arm $i$ is sampled from $N(0, 1)$. For the $\epsilon$-greedy algorithm, we set its exploration probability $\epsilon_{t} = \min\{1,\frac{CN}{t}\}$. We set $C = 500 > 3$ is chosen only for the convenience of presenting the results. In all our experiments, we set the detection method's parameter $\delta$ to $0.05$, the high probability coefficient $\eta$ to $0.05$, and the reward's noise scale $\sigma$ in the environment to $0.1$. We run each experiment for $T= 10000~(\text{for } N =10)$ or $20000~(\text{for } N =30)$ iterations and repeat each experiment $20$ times to report the mean and variance of performance.

\subsection{Experiment Results} 
We first consider the effectiveness of our attack detection method when the attacker applies the commonly used attack algorithm \cite{jun2018adversarial} against UCB1 and $\epsilon$-greedy algorithms. We randomly created $10$ bandit instances to repeat the experiment, where we only vary the ground-truth mean reward gap $\Delta_{1K}$. We set $\Delta_{1K}=\beta(1)/2^{4-i}$ for the $i$-th bandit instance. As shown in Figure \ref{fig:p-detect-UCB1}, with a high probability this attack algorithm's reward manipulations can be detected successfully, which is predicted by Corollary \ref{corollary: UCB1 unattackable} and \ref{corollary: epsilon unattackable}. We find that this result is almost unaffected under different settings of $N$ and $T$. That's because when $N$ is large enough, there may be more non-target arms whose ground-truth mean is significantly larger than the target arm's ground-truth mean.

Next, we compare the performance of the baseline algorithm and our proposed algorithm under the detection. We created different environments to run the experiment, by varying $\Delta_{1K}$ in $\{\beta(1)/2, \beta(1), 2\beta(1)\}$. We stop accumulating the number of target arm pulls once the attack is detected. Note that this does not mean that the learner will necessarily stop learning after this point. From Figure \ref{fig:armpull-UCB1}, we can find that because the attack will be detected quickly, our baseline attack  algorithm cannot trap the victim algorithm to pull the target arm linear times. For our algorithm, because it is stealthy, the victim algorithm failed to notice the reward manipulation and was executed till the end.

\section{Conclusion}
In this paper, we studied the problem of reward poisoning attacks on stochastic multi-armed bandits. We introduced a mechanism to detect such attacks, and we find that previous attack methods against UCB1 and $\epsilon$-greedy algorithms can be easily detected. Focusing on such a detection method, we proposed a stealthy attack method that will succeed under specific conditions concerning the stochastic bandit environments and the reward of the first pulled arm. 

\bibliographystyle{apalike}
\bibliography{main.bib}
\newpage

\appendix
\input{appendix}

\end{document}

%% file: appendix.tex
\onecolumn

\section{Definitions and Prerequisites}
\begin{definition}A real-valued random variable $X$ follows a $\sigma$-sub-Gaussian distribution if
$$
\mathbb{E}[e^{\lambda (X - \mu)}] \leq  e^{\frac{\sigma^2 \lambda^2}{2}}, \quad \forall \lambda \in \mathbb{R}
$$
where $\mu = \mathbb{E}[X]$. 
\end{definition}

\begin{lemma}\label{lemma: hoeffding}  Let $X_{1}, X_{2}, \cdots, X_{n}$ be independent and $\sigma$-sub-Gaussian random variables, $\bar{X}=\frac{1}{n} \sum_{i=1}^{n} X_{i}$, and $\mathbb{E}[X]$ be the mean. Then, for every $\epsilon \geq 0$, we have 
\begin{equation*}
    \mathbb{P}(|\bar{X}-\mathbb{E}[X]| \geq \epsilon) \leq e^{-n \epsilon^{2} / 2 \sigma^{2}} 
\end{equation*}
\end{lemma}

\begin{lemma}\label{lemma: Bernstein} (Bernstein inequality). Let $X_1, \ldots, X_n$ be random variables within range $[0,1]$ and
$$
\sum_{t=1}^n \operatorname{Var}\left[X_t \mid X_{t-1}, \ldots, X_1\right]=\sigma^2
$$
Let $S_n=\sum_{i=1}^{n}X_i$. Then for all $\epsilon \geq 0$
$$
\mathbb{P}\left\{S_n \geq \mathbb{E}\left[S_n\right]+\epsilon\right\} \leq \exp \left\{-\frac{\epsilon^2 / 2}{\sigma^2+\epsilon / 2}\right\}
$$
\end{lemma}

\section{Details of the proofs}

\subsection{Proof of Lemma \ref{lemma: detection error}}

\begin{proof}
Define $E_{\delta}^{t} = \left\{\forall i :\left|\hat{\mu}_{i}^{0}(t)-\mu_{i}\right|<\beta\left(N_{i}(t), \delta\right)\right\}$. By Lemma \ref{lemma:detection basis} the probability of the union of all the type-I errors can be upper bounded by
$
\mathbb{P}\left(\bigcup_{t = 1}^{T}\{\overline{E_{\delta}^{t}} \mid H_{0}^{t}\}\right) = \mathbb{P}\left(\{\bigcup_{t = 1}^{T}\overline{E_{\delta}^{t}}\} \mid H_{0}^{T}\right) = \mathbb{P}\left(\overline{E_{\delta}} \mid H_{0}^{T}\right) \leq \delta
$
\end{proof}

\subsection{Proof of Lemma \ref{lemma: UCB1 unattackable}}
\begin{proof}
We denote $|\Delta_{1K}^{0}-\beta(1)|$ as $\Delta$. We follow this notation throughout the remainder of the appendix. Let $M(T)=\max \left\{\frac{81 \sigma^{2} \log T}{\Delta^{2}}, \frac{3 C(T)}{\Delta}, (\frac{\pi^2N}{3\delta})^{\frac{2}{5}}\right\}$.  Notice that
$$
\begin{aligned}
&\mathbb{P}\left(N_{K}(T) > M(T)\right) \\ &=
\mathbb{P}\left(\bigcup_{t=N+1}^{T}\left\{\text{t is the first time that } N_{K}(t) = M(T)+1\right\}\right)\\&\leq 
\sum_{t=N+1}^{T}\mathbb{P}(\text{t is the first time that } N_{K}(t) = M(T)+1)\\& =
\sum_{t=N+1}^{T}\mathbb{P}(N_{K}(t-1) = M(T), I_{t} = K)\\&\leq
\sum_{t=N+1}^{T} \mathbb{P} \Big (\hat{\mu}_{K}(t-1) + 3\sigma\sqrt{\frac{\log t}{N_{K}(t-1)}} > \hat{\mu}_{1}(t-1) + 3\sigma\sqrt{\frac{\log t}{N_{1}(t-1)}}, {N_{K}(t-1) = M(T)}\Big )
\end{aligned}
$$
where the last step is due to the UCB arm selection strategy.

Notice that we have
$$
\begin{aligned}
&\mathbb{P}\left(\hat{\mu}_{K}(t-1) + 3\sigma\sqrt{\frac{\log t}{N_{K}(t-1)}} > \hat{\mu}_{1}(t-1) + 3\sigma\sqrt{\frac{\log t}{N_{1}(t-1)}}, N_{K}(t-1) = M(T)\right) \\
&\leq \mathbb{P}\left(\hat{\mu}_{K}(t-1) + 3\sigma\sqrt{\frac{\log t}{N_{K}(t-1)}} > \hat{\mu}_{1}(t-1) + 3\sigma\sqrt{\frac{\log t}{N_{1}(t-1)}}\bigg|  N_{K}(t-1) = M(T)\right)
\end{aligned}
$$
Now we bound the terms above. By the Lemma \ref{lemma: hoeffding}, we have for any $w \geq M(T)$, 
\begin{equation}\label{eq:lemma UCB1 unattackable 1}
    \forall i, \mathbb{P}\left(\hat{\mu}_{i}^{0}(t-1)-\mu_{i} \geq \frac{\Delta}{3} \bigg| N_{i}(t-1)=w\right) \leq \frac{1}{T^{9 / 2}}
\end{equation}
Then we have,
$$
\begin{aligned}
&\mathbb{P}\left(\hat{\mu}_{K}(t-1) + 3\sigma\sqrt{\frac{\log t}{N_{K}(t-1)}} > \hat{\mu}_{1}(t-1) + 3\sigma\sqrt{\frac{\log t}{N_{1}(t-1)}}\bigg|  N_{K}(t-1) = M(T)\right) \\
&\quad \leq \mathbb{P}\left(\hat{\mu}_{K}^{0}(t-1)+3 \sigma \sqrt{\frac{\log t}{N_{K}(t-1)}}+\frac{\Delta}{3} \geq \hat{\mu}_{1}(t-1)+3 \sigma \sqrt{\frac{\log t}{N_{1}(t-1)}}  \bigg| N_{K}(t-1) = M(T)\right) \\
&\quad \leq \mathbb{P}\left(\hat{\mu}_{K}^{0}(t-1)+\frac{2\Delta}{3} \geq \hat{\mu}_{1}(t-1)+3 \sigma \sqrt{\frac{\log t}{N_{1}(t-1)}} \bigg| N_{K}(t-1) = M(T)\right)
\end{aligned}
$$
The first inequality holds because $\hat{\mu}_{K}(t-1)-\hat{\mu}_{K}^{0}(t-1) \leq \frac{C(T)}{N_{K}(t-1)} = \frac{C(T)}{M(T)} \leq \frac{\Delta}{3} $ and the second inequality holds because $N_{K}(t-1)\geq$ $\frac{81 \sigma^{2} \log T}{\Delta^{2}}$. We can further upper bound the last term in the above inequality, notice that
$$
\begin{aligned}
&\mathbb{P}\left(\hat{\mu}_{K}^{0}(t-1)+\frac{2\Delta}{3} \geq \hat{\mu}_{1}(t-1)+3 \sigma \sqrt{\frac{\log t}{N_{1}(t-1)}} \bigg|  N_{K}(t-1) = M(T)\right)\\ 
\leq &\mathbb{P}\left(\hat{\mu}_{K}^{0}(t-1)-\mu_{K} \geq \frac{\Delta}{3} \bigg| N_{K}(t-1)=M(T)\right)\\
&+\mathbb{P}\left(\hat{\mu}_{1}(N)-\beta(1)-\hat{\mu}_{1}(t-1) \geq 3 \sigma \sqrt{\frac{\log t}{N_{1}(t-1)}}\right) \\
\leq &\mathbb{P}\left(\hat{\mu}^{0}_{K}(t-1)-\mu_{K} \geq \frac{\Delta}{3} \bigg| N_{K}(t-1)=M(T)\right)\\
&+\mathbb{P}\left( \beta(N_{1}(t-1))\geq 3 \sigma \sqrt{\frac{\log t}{N_{1}(t-1)}}\right)\\
\leq &\frac{1}{T^{9 / 2}}
\end{aligned}
$$
The second inequality holds because the detection in the history of pulls $\mathcal{H}_{t}$ fails to identify the attack, then we have $\hat{\mu}_{1}(N)-\beta(1)- \hat{\mu}_{1}(t-1) \leq \beta(N_{1}(t-1))$. For the third inequality, note that $(\frac{\pi^2N}{3\delta})^{\frac{2}{5}} \leq M(T) = N_{K}(t-1) < t $, we have
$t^{\frac{9}{2}} > \frac{\pi^2Nt^{2}}{3\delta} \geq \frac{\pi^2NN_{1}^{2}(t-1)}{3\delta} $  $\Rightarrow$  $\beta(N_{1}(t-1)) < 3 \sigma \sqrt{\frac{\log t}{N_{1}(t-1)}}$ and Eq \eqref{eq:lemma UCB1 unattackable 1}.
Combining the results above we have
$$
\mathbb{P}\left(N_{K}(T) > M(T)\right) \leq \sum_{t=N+1}^{T} \frac{1}{T^{9 / 2}} \leq \frac{1}{T^{3}} 
$$
This completes the proof.
\end{proof}

\subsection{Proof of Corollary \ref{corollary: UCB1 unattackable}}
\begin{proof}
Consider the attack method described in Corollary \ref{corollary: UCB1 unattackable}, given the environment and $\epsilon$, we can find functions $g(t)$ and $c(t)$ such that the attack method causes the learner to pull the target arm at least $g(t)$ times at a cost no more than $c(t)$ with probability $1-\epsilon$, where $\lim_{t \rightarrow \infty}\frac{g(t)}{t} = 1$ and $\lim_{t \rightarrow \infty}\frac{c(t)}{t} = 0$. Under this environment and attack method, we study the properties of the history of pulls $\mathcal{H}_{T}$.
We refer to $\mathcal{H}_{T}$ as \emph{proper}, if during this game the total cost the attacker spent is no more than $c(T)$ and the learner played target arm at least $g(T)$ times. Then we have $\forall T \geq T_{0}, \mathbb{P}(\text{$\mathcal{H}_{T}$ is proper}) \geq 1-\epsilon  \Rightarrow \mathbb{P}(\text{$\mathcal{H}_{T}$ is not proper}) \leq \epsilon$. Now consider $\mathcal{H}_{T}$ where the attack is not detected and $\Delta^{0}_{1K} > \beta(1) + \Delta_{0}$ holds. By Lemma \ref{lemma: UCB1 unattackable}, with probability at least $1-\frac{1}{T^{3}}$ we have 
$$
N_{K}(T) \leq \max \left\{\frac{3 C(T)}{\Delta_{0}}, \frac{81 \sigma^{2} \log T}{\Delta_{0}^{2}},(\frac{\pi^2N}{3\delta})^{\frac{2}{5}}\right\}
$$
If  $\mathcal{H}_{T}$ is also proper, we have
$$
g(T) \leq \max \left\{\frac{3 C(T)}{\Delta_{0}}, \frac{81 \sigma^{2} \log T}{\Delta_{0}^{2}},(\frac{\pi^2N}{3\delta})^{\frac{2}{5}}\right\}
$$
Notice that let $T \rightarrow \infty$ we have $\frac{g(T)}{T} \rightarrow 1$, $\frac{3 C(T)}{T\Delta_{0}} \rightarrow 0$ and $\frac{81 \sigma^{2} \log T}{T\Delta_{0}^{2}} \rightarrow 0$, so we can find $T_{0}^{'}$ related to $\Delta_{0}$ instead of $T$ such that $\forall T> T_{0}^{'}, g(T) > \max \left\{\frac{3 C(T)}{\Delta_{0}}, \frac{81 \sigma^{2} \ln T}{\Delta_{0}^{2}},(\frac{\pi^2N}{3\delta})^{\frac{2}{5}}\right\}$. 
Let $T>T_{0}^{'}$, we find a contradiction. Hence we have $\forall  T>\max\{T_{0},T_{0}^{'}\} = T_{1}$,
$$
\mathbb{P}( \text{$\mathcal{H}_{T}$ is not proper} \mid \text{attack in } \text{$\mathcal{H}_{T}$ is not detected}, \Delta^{0}_{1K} > \beta(1) + \Delta_{0}) \geq 1-\frac{1}{T^{3}}    \quad \quad 
$$
Combining with $P(\text{$\mathcal{H}_{T}$ is not proper}) \leq \epsilon$, we get
$$
\begin{aligned}
\epsilon \geq \mathbb{P}(\text{$\mathcal{H}_{T}$ is not proper}) &\geq \mathbb{P}(\text{$\mathcal{H}_{T}$ is not proper} ,\text{attack in $\mathcal{H}_{T}$ is not detected}) \\ 
&\geq \mathbb{P}(\text{$\mathcal{H}_{T}$ is not proper} , \text{attack in $\mathcal{H}_{T}$ is not detected}, \Delta^{0}_{1K} > \beta(1) + \Delta_{0}) \\
&\geq (1-\frac{1}{T^{3}})\mathbb{P}(\text{attack in $\mathcal{H}_{T}$ is not detected}, \Delta^{0}_{1K} > \beta(1) + \Delta_{0})
\end{aligned}
$$
Therefore, we have 
$\mathbb{P}(\text{attack in $\mathcal{H}_{T}$ is not detected}, \Delta^{0}_{1K} > \beta(1) + \Delta_{0}) \leq \frac{\epsilon}{(1-\frac{1}{T^{3}})}$.\\
We also note that $\mathbb{P}(\text{attack in $\mathcal{H}_{T}$ is not detected}, \Delta^{0}_{1K} \leq \beta(1) + \Delta_{0}) \leq \mathbb{P}(\Delta^{0}_{1K} \leq \beta(1) + \Delta_{0})$ . Combine these inequalities we have 
 $$\mathbb{P}(\text{attack in $\mathcal{H}_{T}$ is not detected}) \leq \frac{\epsilon}{(1-\frac{1}{T^{3}})} +\mathbb{P}(\hat{\mu}_{1}(N) \leq \mu_K + \beta(1) + \Delta_{0})$$
As a result, $\mathbb{P}(\text{attack in $\mathcal{H}_{T}$ is detected}) \geq 1-\frac{\epsilon}{(1-\frac{1}{T^{3}})}-\mathbb{P}(\hat{\mu}_{1}(N) \leq \mu_K + \beta(1) + \Delta_{0})$. This completes the proof.
\end{proof}

\subsection{Proof of Lemma \ref{lemma: epsilon unattackable}}
\begin{proof}
Let $M(T)=\max \left\{e^{\frac{5a}{9}-1}C^{2}N^{2}\log T, \frac{81 \sigma^{2} \log T}{\Delta^{2}}, \frac{3 C(T)}{\Delta}, CNe^{\frac{a}{6C}-\frac{1}{2}}\right\} $. Define $E^{'}=\{\exists t_{0} \text{  s.t. } N_{K}(t_{0}) \geq M(T) \cap \exists t_{1} \geq t_{0} \text{ s.t.} \text{ the} \text{ learner} \text{ chooses} \text{ to} \text{ exploit} \text{ at time } t_{1} \text{ and } I_{t_{1}}=K\} $.
Notice that
$$
\begin{aligned}
&\mathbb{P}\left(E^{'}\right) \\& =
\mathbb{P}\left(\bigcup_{t=1}^{T}\{t \text{ is the first time the learner chooses to exploit  and } N_{K}(t) = M(T)+1\}\right)\\&\leq 
\sum_{t > M(T)}^{T}\mathbb{P}\left(t \text{ is the first time the learner chooses to exploit  and } N_{K}(t) = M(T)+1\right)\\& \leq
\sum_{t > M(T)}^{T} \mathbb{P}\left(\hat{\mu}_{K}(t-1)  > \hat{\mu}_{1}(t-1), N_{K}(t-1) = M(T)\right) \\
& \leq
\sum_{t > M(T)}^{T} \mathbb{P}\left(\hat{\mu}_{K}(t-1)  > \hat{\mu}_{1}(t-1) | N_{K}(t-1) = M(T)\right) 
\end{aligned}
$$
By Lemma \ref{lemma: hoeffding}, we have for any $u \geq M(T)$ 
\begin{equation}\label{eq: lemma epsilon unattackable 1}
    \forall i, \mathbb{P}\left(\hat{\mu}_{i}^{0}(t-1)-\mu_{i} \geq \frac{\Delta}{3} \bigg| N_{i}(t-1)=u\right) \leq \frac{1}{T^{\frac{9}{2}}}
\end{equation}
Then 
$$
\begin{aligned}
&\mathbb{P}\left(\hat{\mu}_{K}(t-1) > \hat{\mu}_{1}(t-1) \mid  N_{K}(t-1) = M(T)\right) \\
\leq &\mathbb{P}\left(\hat{\mu}_{K}^{0}(t-1)+\frac{\Delta}{3} > \hat{\mu}_{1}(t-1) \bigg| N_{K}(t-1) = M(T)\right) \\
\leq &\mathbb{P}\left(\hat{\mu}_{K}^{0}(t-1)-\mu_{K} > \frac{\Delta}{3} \bigg| N_{K}(t-1)=M(T)\right)+\\
&\mathbb{P}\left(\hat{\mu_{1}}(N)-\beta(1) \geq \hat{\mu}_{1}(t-1) + \frac{\Delta}{3} \bigg| t > M(T) \right) \\
\leq &\frac{1}{T^{\frac{9}{2}}} + \mathbb{P}\left(\beta(N_{1}(t-1)) \geq \frac{\Delta}{3} \bigg| t > M(T) \right)
\end{aligned}
$$
The first inequality holds because $\hat{\mu}_{K}(t-1)-\hat{\mu}_{K}^{0}(t-1) \leq \frac{C(T)}{N_{K}(t-1)} = \frac{C(T)}{M(T)} \leq \frac{\Delta}{3} $. The second inequality holds because $\hat{\mu}_{1}(N)-\beta(1)- \hat{\mu}_{1}(t-1)\leq \beta(N_{1}(t-1))$ and Eq \eqref{eq: lemma epsilon unattackable 1}. 

Define $a$ as a constant that does not depend on $T$, such that $\beta(a) = \frac{\Delta}{3}$, or $a = \beta^{-1}(\frac{\Delta}{3})$. Similar to \cite{auer2002finite}, let $N_{i}^{R}(t)$ be the number of plays in which arm $i$ was chosen at random in the first $t$ rounds. Define $x_{0} = \frac{1}{2 N} \sum_{i=1}^t \varepsilon_i$, by Bernstein's inequality we have,
$$
\begin{aligned}
    \mathbb{P}\left(\beta(N_{1}(t-1)) \geq \frac{\Delta}{3} \bigg| t > M(T) \right) 
    &\leq \mathbb{P}\left(N_{1}(t-1) \leq a \bigg| t > M(T) \right) \\ &\leq \mathbb{P}\left(N^{R}_{1}(t-1) \leq a \bigg| t > M(T) \right) 
    \\ &\leq \exp\left(-\frac{(2x_{0}-a)^{2}}{6x_{0}-a}\right), t > M(T)
\end{aligned}
$$
The last term equals to $\exp\left(-\frac{2x_{0}}{3}+\frac{5a}{9}-\frac{\frac{4a^{2}}{9}}{6x_{0}-a}\right)$. Note that $x_{0} \geq C\log \frac{te^{\frac{1}{2}}}{CN}$ and $t > M(T)$, we have $6x_{0}-a > 0$. Then we can bound 
$$
\mathbb{P}\left(\beta(N_{1}(t-1)) \geq \frac{\Delta}{3}\bigg| t > M(T) \right) \leq \exp\left(-\frac{2x_{0}}{3}+\frac{5a}{9}\right) \leq e^{\frac{5a}{9}}\left(\frac{CN}{te^{\frac{1}{2}}}\right)^\frac{2C}{3}
$$ 

Note that $C > 3$, combining the inequalities above we have
\begin{align*}
\label{eq:lemma epsilon unattackable 2}
    \mathbb{P}\left(E^{'}\right) &\leq \sum_{t > M(T)}^{T} \mathbb{P}\left(\hat{\mu}_{K}(t-1)  > \hat{\mu}_{i}(t-1) | N_{K}(t-1) = M(T)\right) \\ 
    &\leq \frac{1}{T^{\frac{7}{2}}} + \sum_{t > M(T)}^{T} e^{\frac{5a}{9}}\left(\frac{CN}{te^{\frac{1}{2}}}\right)^2 \nonumber \\
    &\leq \frac{1}{T^{\frac{7}{2}}} + \frac{1}{\log T}\nonumber
\end{align*}
The second inequality holds because $\sum_{t > M(T)}^{T}\frac{1}{t^{2}} \leq \frac{1}{M(T)}$. When $\overline{E^{'}}$ happens, define $Y_{t}$ as the Bernoulli random variable for the event that at round $t$ arm $K$ is pulled, then we can bound $N_{K}(T)$ in the following way:
\begin{equation}\label{eq:lemma epsilon unattackable 3}
 N_{K}(T) \leq M(T) + \sum_{t \geq 1}Y_{i}   
\end{equation}
Notice that $\sum_{t=1}^{T}\frac{\epsilon_{t}}{N} \leq C + C\log \frac{T}{CN}$, and from the proof of Lemma 4 in \cite{jun2018adversarial}, with probability at least $1-\eta$ we have 
\begin{equation}\label{eq:lemma epsilon unattackable 4}
    \sum_{t \geq 1}Y_{t} \leq \sum_{t=1}^T \frac{\epsilon_t}{N}+\sqrt{3 \sum_{t=1}^T \frac{\epsilon_t}{N} \log \frac{1}{\eta}} \leq C + C\log\frac{T}{CN} + \sqrt{3(C + C\log\frac{T}{CN})\log \frac{1}{\eta}}
\end{equation}
Combining Eq (\ref{eq:lemma epsilon unattackable 3}) and (\ref{eq:lemma epsilon unattackable 4}), with probability at least $1-\eta-\frac{1}{T^{3.5}}-\frac{1}{\log T}$, we have 
$$
N_{K}(T) \leq M(T) + C + C\log\frac{T}{CN} + \sqrt{3(C + C\log\frac{T}{CN})\log \frac{1}{\eta}}
$$
\end{proof}

\subsection{Proof of Corollary \ref{corollary: epsilon unattackable}}
\begin{proof}
Consider the attack method as described in this corollary, and we follow the notation and structure from the proof of Corollary \ref{corollary: UCB1 unattackable} in this proof. Again, we consider the attack in $\mathcal{H}_{T}$ is not detected and $\Delta^{0}_{1K} > \beta(1) + \Delta_{0}$. By Lemma \ref{lemma: epsilon unattackable}, with probability $1-\eta-\frac{1}{T^{3.5}}-\frac{1}{\log T}$ we have 
$$
\begin{aligned}
N_{K}(T) \leq &\max \left\{C_1\log T, \frac{81 \sigma^{2} \log T}{\Delta_{0}^{2}}, \frac{3 C(T)}{\Delta_{0}}, C_2\right\}+ C + C\log\frac{T}{CN}  \\&+\sqrt{3(C + C\log\frac{T}{CN})\log \frac{N}{\eta}}
\end{aligned}
$$
where $C_1 = e^{\frac{5a}{9}-1}C^{2}N^{2}$ and $C_2 = CNe^{\frac{a}{6C}-\frac{1}{2}}$.
If  $\mathcal{H}_{T}$ is also proper, we further have
$$
\begin{aligned}
    g(T) \leq & \max \left\{C_1\log T, \frac{81 \sigma^{2} \log T}{\Delta_{0}^{2}}, \frac{3 C(T)}{\Delta_{0}}, C_2\right\}+ C + C\log\frac{T}{CN}  \\
    &+\sqrt{3(C + C\log\frac{T}{CN})\log \frac{N}{\eta}}
\end{aligned}
$$
Notice that when $T \rightarrow \infty$ we have $\frac{g(T)}{T} \rightarrow 1$, $\frac{C_1\log T}{T} \rightarrow 0$, $\frac{81 \sigma^{2} \log T}{T\Delta_{0}^2} \rightarrow 0$, $\frac{3 C(T)}{T\Delta_{0}} \rightarrow 0$ and $C\dfrac{\log \frac{T}{CN}}{T} \rightarrow 0$. Hence, we can find $T_{0}^{'}$ related to $\Delta_{0}$ instead of $T$ such that $\forall T>T_{0}^{'}$, 
$g(T) > \max \left\{C_1\log T, \frac{81 \sigma^{2} \log T}{\Delta_{0}^{2}}, \frac{3 C(T)}{\Delta_{0}}, C_2\right\}+ C + C\log\frac{T}{CN} + \sqrt{3(C + C\log\frac{T}{CN})\log \frac{N}{\eta}}$. Let $T>T_{0}^{'}$, we find a contradiction. Hence, for $\forall  T>\max\{T_{0},T_{0}^{'}\} = T_{1}$, we have the following,
$$
\mathbb{P}( \text{$\mathcal{H}_{T}$ is not proper} \mid \text{attack in $\mathcal{H}_{T}$ is not detected}, \Delta^{0}_{1K} > \beta(1) + \Delta_{0}) \geq 1-\eta-\frac{1}{T^{3.5}}-\frac{1}{\log T}    \quad \quad 
$$
Then following a similar derivation as in the proof of Corollary \ref{corollary: UCB1 unattackable}, we can prove that: 
 $$\mathbb{P}(\text{attack in $\mathcal{H}_{T}$ is not detected}) \leq \frac{\epsilon}{(1-\eta-\frac{1}{T^{3.5}}-\frac{1}{\log T})} + \mathbb{P}(\Delta^{0}_{1K} \leq \beta(1) + \Delta_{0})$$
In other words, $\mathbb{P}(\text{attack in $\mathcal{H}_{T}$ is detected}) \geq 1-\frac{\epsilon}{(1-\eta-\frac{1}{T^{3.5}}-\frac{1}{\log T})}-\mathbb{P}(\hat{\mu}_{1}(N) \leq \mu_K + \beta(1) + \Delta_{0})$. This completes the proof.
\end{proof}

\subsection{Proof of Theorem \ref{theorem: UCB1 attack}}
\begin{proof}
By Lemma \ref{lemma: hoeffding}, we have $\forall t \geq N$
\begin{equation}\label{eq: UCB1 attack 1} \mathbb{P}\left(\mu_{K} - \hat{\mu}_{K}^{0}(t-1) \geq 3\sigma\sqrt{\frac{\log t}{N_{K}(t-1)}}\right) \leq \frac{1}{t^{9 / 2}}
\end{equation}
Define $$
E_{2}:=\left\{ \forall t > \overline{t}:\mu_{K} - \hat{\mu}_{K}^{0}(t-1) < 3\sigma\sqrt{\frac{\log t}{N_{K}(t-1)}}\right\}
$$ 
Where $\overline{t} = \lfloor \frac{9\sigma^{2}\log T}{(\Delta^{0}_{1K}-\beta(1))^{2}} \rfloor$. Then by Eq \eqref{eq: UCB1 attack 1} we have 
$$\mathbb{P}(\overline{E_{2}}) \leq 
\sum_{t > \overline{t}}\frac{1}{t^{9 / 2}} \leq \frac{2(\Delta^{0}_{1K}-\beta(1))^{7}}{7(9\sigma^{2}\log T)^{\frac{7}{2}}} \Longrightarrow \mathbb{P}({E_{2}}) \geq 1 - \frac{2(\Delta^{0}_{1K}-\beta(1))^{7}}{7(9\sigma^{2}\log T)^{\frac{7}{2}}} $$ 
For $\eta \in (0,1)$, define 
$$
E_{\eta}:=\left\{\forall i, \forall t \geq N:\left|\hat{\mu}_{i}^{0}(t)-\mu_{i}\right|<\beta\left(N_{i}(t),\eta\right)\right\}
$$
We have that 
$$
\mathbb{P}(E_{\eta} \cap E_{2}) \geq 1 - \eta - \frac{2(\Delta^{0}_{1K}-\beta(1))^{7}}{7(9\sigma^{2}\log T)^{\frac{7}{2}}}
$$
Without loss of generality, we assume that in round $t \geq \overline{t}$ arm $i$ is pulled. If $i=1$, assume that the last time arm $1$ was pulled was $t^{'}$ and $N_{1}(t^{'}) \geq 1$. Suppose events $E_{2}$ and $E_{\eta}$ both hold; based on design of the UCB1 algorithm, we have
$$
\hat{\mu}_{1}(t-1)+3 \sigma \sqrt{\frac{\log t}{N_{1}(t-1)}} \geq \hat{\mu}_{K}(t-1)+3 \sigma \sqrt{\frac{\log t}{N_{K}(t-1)}}
$$
Based on the attack method we have $\hat{\mu}_{1}(t-1) = \hat{\mu}_{1}(t^{'}) = \hat{\mu}_{1}(N) - \beta(1) - \beta(N_{1}(t^{'}))$.
Since event $E_{2}$ holds, we have 
$$
\begin{aligned}
\hat{\mu}_{1}(N) - \beta(1) + 3\sigma\sqrt{\frac{\log t}{N_{1}(t^{'})}} 
&>\hat{\mu}_{1}(N) - \beta(1) - \beta(N_{1}(t^{'}))+3\sigma\sqrt{\frac{\log t}{N_{1}(t^{'})}} \\
&=
\hat{\mu}_{1}(t^{'})+3\sigma\sqrt{\frac{\log t}{N_{1}(t^{'})}} \\&=
\hat{\mu}_{1}(t-1)+3\sigma\sqrt{\frac{\log t}{N_{1}(t-1)}}   \\&\geq 
\hat{\mu}_{K}(t-1)+3 \sigma \sqrt{\frac{\log t}{N_{K}(t-1)}} \\&>
\mu_{K}
\end{aligned}
$$
This shows that 
\begin{equation}\label{eq:lemma UCB1 attack 1}
N_{1}(t) = N_{1}(t^{'}) +1 \leq  \frac{9\sigma^{2}\log t}{(\Delta^{0}_{1K}-\beta(1))^{2}} + 1   
\end{equation}
We also have
$$
\begin{aligned}
\alpha_{t} &=\hat{\mu}_{1}(t-1) N_{1}(t-1)+r_{t}^{0}-\left(\hat{\mu}_{1}(N)-\beta(1)- \beta\left(N_{1}(t)\right)\right) N_{1}(t) \\
&=\sum_{\left\{s: s \le t-1, I_{s}=i\right\}}\left(r_{s}^{0}-\alpha_{s}\right)+r_{t}^{0}-\left(\hat{\mu}_{1}(N)-\beta(1)-\beta\left(N_{1}(t)\right)\right) N_{1}(t) \\
&=\sum_{\left\{s: s \le t, I_{s}=i\right\}} r_{s}^{0}-\sum_{\left\{s: s \le t-1, I_{s}=i\right\}} \alpha_{s}-\left(\hat{\mu}_{1}(N)-\beta(1)- \beta\left(N_{1}(t)\right)\right) N_{1}(t) .
\end{aligned}
$$
Therefore, the cumulative attack on arm $1$ is
$$
\begin{aligned}
\sum_{\left\{s: s \le t, I_{s}=i\right\}} \alpha_{s} &=\sum_{\left\{s: s \le t, I_{s}=i\right\}} r_{s}^{0}-\left(\hat{\mu}_{1}(N)-\beta(1)- \beta\left(N_{1}(t)\right)\right) N_{1}(t) \\
&=(\hat{\mu}_{1}^{0}(t)-\hat{\mu}_{1}(N)+\beta(1)+ \beta\left(N_{1}(t)\right)) N_{1}(t)
\end{aligned}
$$
Since event $E_{\eta}$ holds,
$$
\begin{aligned}
\hat{\mu}_{1}^{0}(t) &<\hat{\mu}_{1}(N)+\beta\left(N_{1}(t),\eta \right)+\beta(1, \eta) 
\end{aligned}
$$
Therefore, by Eq (\ref{eq:lemma UCB1 attack 1})
$$
\begin{aligned}
\sum_{\left\{s: s \le t, I_{s}=i\right\}} \alpha_{s} &\leq (\beta(1)+\beta(N_{1}(t))+\beta\left(N_{1}(t),\eta \right)+\beta(1, \eta))(1+\frac{9\sigma^{2}\log T}{(\Delta^{0}_{1K}-\beta(1))^{2}}) \\
&\leq 2(\beta(1)+\beta(1, \eta))(1+\frac{9\sigma^{2}\log T}{(\Delta^{0}_{1K}-\beta(1))^{2}}).
\end{aligned}
$$
If $i \neq 1$, suppose at time $t$ the learner chooses arm $i \neq 1 \text{ or } K$. Then we have
$$
\hat{\mu}_{i}(t-1)+3 \sigma \sqrt{\frac{\log t}{N_{i}(t-1)}} > \hat{\mu}_{1}(t-1)+3 \sigma \sqrt{\frac{\log t}{N_{1}(t-1)}}
$$
Since event $E_{\eta}$ holds, combined with the way that the attack algorithm attacks $\hat{\mu}_{i}^{0}(N)$, we can obtain
$$
\begin{aligned}
3 \sigma \sqrt{\frac{\log t}{N_{i}(t-1)}} &> \hat{\mu}_{1}(t-1) - \hat{\mu}_{i}(t-1) \\
&\geq (\hat{\mu}_{1}(N) -\beta(1)-\beta(N_{1}(t-1)))- (\hat{\mu}_{i}(N)+\beta(1,\eta)+\beta(N_{i}(t-1)),\eta))\\
&>\hat{\mu}_{1}(N) - \hat{\mu}_{i}(N)-2\beta(1,\eta)-2\beta(1)\\
&\geq d
\end{aligned}
$$

As a result, we have
\begin{equation}\label{eq:lemma UCb1 attack 2} 
N_{i}(t) \leq N_{i}(t-1) + 1 \leq 1 + \frac{9\sigma^{2}\log t}{d^2}
\end{equation}
Note that if $\hat{\mu}_{i}^{0}(N) \leq \hat{\mu}_{1}(N) - 2\beta(1,\eta) -2\beta(1) -d$ then $\alpha_{i} = 0$; Otherwise,
$$
\begin{aligned}
\alpha_{i} &=|\hat{\mu}_{i}(N)-\hat{\mu}^{0}_{i}(N)| \\
&\leq 
|\hat{\mu}_{i}(N)-\hat{\mu}_{1}(N)|+|\hat{\mu}_{1}(N)-\hat{\mu}^{0}_{i}(N)| \\
& \leq 2\beta(1) + 2\beta(1,\eta)+d +|\hat{\mu}_{1}(N)-\mu_{1}|+|\mu_{1}-\mu_{i}|+ |\mu_{i}-\hat{\mu}^{0}_{i}(N)| \\
&\leq |\mu_{1}-\mu_{i}|+2\beta(1) + 4\beta(1,\eta)+d
\end{aligned}
$$  
The attack cost is bounded by 
$$
\sum_{\left\{s: s \le t, I_{s}=i\right\}} \alpha_{s} = \alpha_{i}N_{i}(t) \leq N_{i}(t)(|\Delta_{1i}|+2\beta(1) +4\beta(1,\eta)+d)
$$
Therefore, we have with probability $\mathbb{P}(E_{\eta} \cap E_{2}) \geq 1 - \eta - \frac{2(\Delta^{0}_{1K}-\beta(1))^{7}}{7(9\sigma^{2}\log T)^{\frac{7}{2}}}$ 
and the following,
$$ 
\begin{aligned}
\sum_{i=1}^{T}\alpha_{s} =& \sum_{i=1}^{N}\sum_{s: I_{s}=i}^{t} \alpha_{s} \\
\leq& 2(\beta(1)+\beta(1, \eta))(1+\frac{9\sigma^{2}\log T}{(\Delta^{0}_{1K}-\beta(1))^{2}}) \\&+
(\frac{9\sigma^{2}\log T}{d^2}+1)\sum_{i \neq 1,K}(|\Delta_{1i}|+2\beta(1)+4\beta(1,\eta)+d) \\ 
\leq &
\left[\frac{18(\beta(1)+\beta(1, \eta))\sigma^{2}}{(\Delta^{0}_{1K}-\beta(1))^{2}} + \frac{9N\sigma^{2}(2\beta(1)+4\beta(1,\eta)+d)}{d^2}\right]\log T  \\&+ \left(\frac{9\sigma^{2}\log T}{d^2}+1\right)\sum_{i \neq 1,K}\Delta_{1i} + 4N\beta(1,\eta) + 2N\beta(1) + 2\beta(1) + dN
\end{aligned}
$$
Now consider the number of times the target arm will be pulled. By Eq (\ref{eq:lemma UCB1 attack 1}) and (\ref{eq:lemma UCb1 attack 2}), with probability at least $1 - \eta - \frac{2(\Delta^{0}_{1K}-\beta(1))^{7}}{7(9\sigma^{2}\log T)^{\frac{7}{2}}}$ the learner will pull the target arm $K$ at least 
$$
\begin{aligned}
    T-\sum_{i=1}^{N}N_{i}(T) &\geq T-\frac{9\sigma^{2}\log T}{(\Delta^{0}_{1K}-\beta(1))^{2}}-\frac{9N\sigma^{2}\log T}{d^2} -(N-1) \\ &= T-\left(\frac{9\sigma^{2}}{(\Delta^{0}_{1K}-\beta(1))^2}+\frac{9N\sigma^{2}}{d^2}\right)\log T -(N-1)
\end{aligned}
$$rounds. This completes the proof.
\end{proof}

\subsection{Proof of Theorem \ref{theorem: epsilon attack}}
\begin{proof}
Without loss of 
    generality, assume in round $t$ the learner chooses to exploit. Suppose event $E_{\eta}$ holds. Then for $i \neq 1$, we have 
$$
\begin{aligned}
 &\hat{\mu}_{1}(t-1) - \hat{\mu}_{i}(t-1) \\
&\geq (\hat{\mu}_{1}(N) -\beta(1)-\beta(N_{1}(t-1)))- (\hat{\mu}_{i}(N)+\beta(1,\eta)+\beta(N_{i}(t-1)),\eta))\\
&>\hat{\mu}_{1}(N) - \hat{\mu}_{i}(N)-2\beta(1,\eta) - 2\beta(1)\\
&\geq d
\end{aligned}
$$
which means the learner will not choose arm $i \neq 1 \text{ or } K$ when round $t$ is an exploitation round.

For $i = 1$, if $I_{t} = i$ we assume that the last time arm $1$ was pulled was $t^{'}$ and $N_{1}(t^{'}) \geq 1$. According to the $\epsilon$-greedy algorithm, we have 
\begin{align*}
     \hat{\mu}_{1}(N)-\beta(1)-\beta(N_{1}(t^{'})) = \hat{\mu}_{1}(t-1) &> \hat{\mu}_{K}(t-1) \\
\Rightarrow  
\mu_{K} - \hat{\mu}_{K}(t-1) &> \beta(N_{1}(t^{'})) + \Delta \\
\Rightarrow 
\mu_{K} - \hat{\mu}_{K}(t-1) &> \Delta
\end{align*}

Define $\overline{N}_{i}(t)$ as the number of arm $i$ pulls in the exploitation rounds before time $t$.
$$
\begin{aligned}
    \mathbb{P}(\overline{N}_{1}(T) > t_1 + (t_2-t_1) + 0) \leq  \mathbb{P}(N_1 > t_1) + \mathbb{P}(N_2 > t_2 - t_1) + \mathbb{P}(N_3 > 0) = \mathbb{P}(N_3 > 0)
\end{aligned}
$$
where $N_1, N_2$ and $N_3$ are the number of arm pulls in the exploitation rounds before $t_1$, during $[t_1 + 1, t_2]$, and after $t_2$, accordingly. \\ 
$$
\begin{aligned}
\mathbb{P}(N_3 > 0) &\leq \mathbb{P}(\exists t> t_2, \mu_K - \hat{\mu}_K(t-1) > \Delta) \\
\leq& \mathbb{P}(\exists t> t_2, \mu_K - \hat{\mu}_K(t-1) > \Delta \mid N_{K}(t_2) \geq \frac{2\sigma^2}{\Delta^2}\log\frac{4T}{\eta})  \\ &+\mathbb{P}(N_{K}(t_2) < \frac{2\sigma^2}{\Delta^2}\log\frac{4T}{\eta} \mid N^{R}_{K}(t_1) \geq \frac{2\sigma^2}{\Delta^2}\log (1+\frac{\Delta^2}{\sigma^2}), E_{exp}) \\
&+\mathbb{P}(N^{R}_{K}(t_1) < \frac{2\sigma^2}{\Delta^2}\log (1+\frac{\Delta^2}{\sigma^2})) + \mathbb{P}(\overline{E_{exp}})
\end{aligned}
$$
where 
$$
\begin{aligned}
E_{exp} = &\left\{{\sum_{i > t_1}^{t_2} Y_i \leq \sum_{i > t_1}^{t_2} \epsilon_i + \sqrt{3\sum_{i > t_1}^{t_2} \epsilon_i\log\frac{4}{\eta}}} \right\}
\end{aligned}
$$ and $Y_t$ is the Bernoulli random variable indicate whether round $t$ is for exploration.

By the proof of Lemma \ref{lemma: epsilon unattackable} (denote $\frac{2\sigma^2}{\Delta^2}\log (1+\frac{\Delta^2}{\sigma^2})$ as $b$, then $b < 2$) we have for $t > CNe^{\frac{b}{6C}-\frac{1}{2}}$ 
$$
\mathbb{P}\left(N^{R}_{K}(t) < b\right) \leq e^{\frac{5b}{9}}\left(\frac{CN}{te^{\frac{1}{2}}}\right)^\frac{2C}{3}
$$ 
Therefore, by setting $t_1 = \max\{CNe^{\frac{b}{6C}-\frac{1}{2}}, CNe^{\frac{5b}{6C}-\frac{3}{2C}\log\frac{\eta}{4} - \frac{1}{2}}, 5CN\}$, we have 
$\mathbb{P}\left(N^{R}_{K}(t_1) < b\right) < \frac{\eta}{4}$. 

Next, notice that $\mathbb{P}(E_{exp}) \geq 1-\frac{\eta}{4} \Rightarrow \mathbb{P}(\overline{E}_{exp}) \leq \frac{\eta}{4}$. And for $(Y_{t_1 + 1}, \cdots, Y_{t_2}) \in E_{exp}$ we have 
$$
\sum_{i > t_1}^{t_2} Y_i \leq \sum_{i > t_1}^{t_2} \epsilon_i + \sqrt{3\sum_{i > t_1}^{t_2} \epsilon_i\log\frac{4}{\eta}} \leq \frac{5}{2}\sum_{i > t_1}^{t_2} \epsilon_i + \log\frac{4}{\eta} \leq \frac{5CN}{2}\log(\frac{t_2}{t_1}) + \log\frac{4}{\eta}
$$
Then the number of exploitation rounds in time interval $[t_1 + 1, t_2]$ (denote as $N_{exp}$) is larger than
$t_2 - t_1 - \frac{5CN}{2}\log(\frac{t_2}{t_1}) - \log\frac{4}{\eta} \geq t_2 - t_1 - \frac{5CN}{2}(\frac{t_2}{t_1 }-1) - \log\frac{4}{\eta} \geq \frac{1}{2}(t_2-t_1) - \log\frac{4}{\eta}$.
If $t$ is an exploitation round, define $W_t$ as the Bernoulli random variable suggesting at round $t$ arm $1$ was chosen. Notice that
\begin{equation}
 \mathbb{P}\left(\mu_{K} - \hat{\mu}_{K}^{0}(t-1) > \Delta \bigg| N_{K}(t-1) \geq \frac{2\sigma^2}{\Delta^2}\log (1 + \frac{\Delta^2}{\sigma^2}))\right) \leq \frac{1}{1 + \frac{\Delta^2}{\sigma^2}}
\end{equation}
If $N_{K}(t-1) \geq \frac{2\sigma^2}{\Delta^2}\log (1 + \frac{\Delta^2}{\sigma^2})$ then $\mathbb{P}(W_t = 1) \leq \frac{1}{1 + \frac{\Delta^2}{\sigma^2}}$. Denote this probability as $\omega_t$. Then we have with probability at least $1-\frac{\eta}{4}$
$$
\sum_{\substack{t_2 \geq t > t_1\\ t \text{ is exploitation round}}}W_t \leq \sum_{\substack{t_2 \geq t > t_1\\ t \text{ is exploitation round}}}\omega_t + \sqrt{3\sum_{\substack{t_2 \geq t > t_1\\ t \text{ is exploitation round}}}\omega_t\log\frac{4}{\eta}} $$ 
As a result,
$$
\begin{aligned}
 N_{K}(t_2) \geq \sum_{\substack{t_2 \geq t > t_1\\ t \text{ is exploitation round}}}(1-w_t) &\geq N_{exp}-\sum_{\substack{t_2 \geq t > t_1\\ t \text{ is exploitation round}}}\omega_t-\sqrt{3\sum_{\substack{t_2 \geq t > t_1\\ t \text{ is exploitation round}}}\omega_t\log\frac{4}{\eta}} \\ &\geq (1-\frac{1}{1 + \frac{\Delta^2}{\sigma^2}})N_{exp} - \sqrt{\frac{3N_{exp}}{1 + \frac{\Delta^2}{\sigma^2}}\log\frac{4}{\eta}}
\end{aligned}
$$
Then it is clear that
\begin{align*}
&N_{K}(t_2) \geq \frac{2\sigma^2}{\Delta^2}\log\frac{4T}{\eta}  \\
\Leftarrow~~ &
(1-\frac{1}{1 + \frac{\Delta^2}{\sigma^2}})N_{exp} - \sqrt{\frac{3N_{exp}}{1 + \frac{\Delta^2}{\sigma^2}}\log\frac{4}{\eta}} \geq \frac{2\sigma^2}{\Delta^2}\log\frac{4T}{\eta}  \\
\Leftarrow~~ &
N_{exp} \geq 2(1 + \frac{\sigma^2}{\Delta^2})(\frac{3\sigma^2}{4\Delta^2}\log \frac{4}{\eta} + \frac{2\sigma^2}{\Delta^2}\log\frac{4T}{\eta}) \\
\Leftarrow~~ &
t_2 \geq 4(1 + \frac{\sigma^2}{\Delta^2})(\frac{3\sigma^2}{4\Delta^2}\log \frac{4}{\eta} + \frac{2\sigma^2}{\Delta^2}\log\frac{4T}{\eta}) + t_1 + 2\log\frac{4}{\eta}
\end{align*}

Let $t_2 = 4(1 + \frac{\sigma^2}{\Delta^2})(\frac{3\sigma^2}{4\Delta^2}\log \frac{4}{\eta} + \frac{2\sigma^2}{\Delta^2}\log\frac{4T}{\eta}) + 2\log\frac{4}{\eta} + t_1$, we have $\mathbb{P}(N_{K}(t_2) < \frac{2\sigma^2}{\Delta^2}\log\frac{4T}{\eta} \mid N^{R}_{K}(t_1) \geq \frac{2\sigma^2}{\Delta^2}\log (1+\frac{\Delta^2}{\sigma^2}), E_{exp}) < \frac{\eta}{4}$ 

Finally, according to Lemma \ref{lemma: hoeffding}, we have $\forall t \geq t_2$
\begin{equation}
\mathbb{P}\left(\mu_{K} - \hat{\mu}_{K}^{0}(t-1) \geq \Delta \bigg| N_{K}(t_2) \geq \frac{2\sigma^2}{\Delta^2}\log\frac{4T}{\eta} \right) \leq \frac{\eta}{4T}
\end{equation}
\begin{equation}
\mathbb{P}(\exists t\geq t_2, \mu_K - \hat{\mu}_K(t-1) > \Delta \mid N_{K}(t_2) \geq \frac{2\sigma^2}{\Delta^2}\log\frac{4T}{\eta}) \leq \frac{\eta}{4}
\end{equation}
Combine the above inequalities we have 
$$
\mathbb{P}(N_3 > 0 ) \leq \eta
$$
$$
\Rightarrow \mathbb{P}(\overline{N}_{1}(T) > t_2) \leq \eta
$$

Suppose both $E_{\eta}$ and $\overline{N}_{1}(T) > t_2$ holds, we have with probability at least $1-\eta$
$$
N_{i}(T) < \sum_{t=1}^T \frac{\epsilon_t}{N}+\sqrt{3 \sum_{t=1}^T \frac{\epsilon_t}{N} \log \frac{N}{\eta}} \leq C + C\log\frac{T}{CN} + \sqrt{3(C + C\log\frac{T}{CN})\log \frac{N}{\eta}}
$$
$$
N_{1}(T)<t_{2} + \sum_{t = 1}^T \frac{\epsilon_t}{N}+\sqrt{3 \sum_{t = 1}^T \frac{\epsilon_t}{N} \log \frac{N}{\eta}}
\leq t_{2} + C + C\log\frac{T}{CN} + \sqrt{3(C + C\log\frac{T}{CN})\log \frac{N}{\eta}}$$
$$
\begin{aligned}
    N_{K}(T) &> T-t_{2} - (N-1)/N\sum_{t = 1}^T\epsilon_t-(N-1)\sqrt{3 \sum_{t = 1}^T \epsilon_t\log \frac{N}{\eta}} \\ &\geq T - t_{2} - (N-1)\left(C + C\log\frac{T}{CN} + \sqrt{3(C + C\log\frac{T}{CN})\log \frac{1}{\eta}}\right)
\end{aligned}
$$
Next, we bound the attack cost. From the proof of Lemma \ref{lemma: epsilon unattackable} we know that the cumulative attack on arm $1$ can be bound by 
$$
\sum_{s: I_{s}=1}^{t} \alpha_{s} =(\hat{\mu}_{1}^{0}(t)-\hat{\mu}_{1}(N)+\beta(1)+ \beta\left(N_{1}(t)\right)) N_{1}(t) \leq 2(\beta(1)+\beta(1,\eta))N_{1}(t)
$$
And 
$$
\sum_{s: I_{s}=i}^{t} \alpha_{s} = N_{i}(T)\alpha_{i} \leq N_{i}(t)(|\Delta_{1i}|+2\beta(1)+4\beta(1,\eta)+d)
$$
Then with probability at least $1 - 3\eta$, 
$$
\begin{aligned}
\sum_{i=1}^{T}\alpha_{s} = \sum_{i=1}^{N}\sum_{s: I_{s}=i}^{t} \alpha_{s} &\leq 2(\beta(1) + \beta(1, \eta))N_{1}(T) + \sum_{i \neq 1,K}N_{i}(T)(|\Delta_{1i}|+2\beta(1)+4\beta(1, \eta)+d) \\ 
&\leq
2(\beta(1) + \beta(1, \eta))t_{2} + \left(C + C\log\frac{T}{CN} + \sqrt{3(C + C\log\frac{T}{CN})\log \frac{N}{\eta}}\right)\\&(4N\beta(1, \eta) + 2N\beta(1) + \sum_{i \neq 1,K}|\Delta_{1i}| + dN)
\end{aligned}
$$
This completes the proof. 
\end{proof}

\subsection{Proof of Lemma \ref{lemma: construction}}
\begin{proof}
We design our general bandit algorithm and corresponding attack method as follows. Let $\mathcal{T}$ be a set of integers that we will decide later and $\mathcal{T}_{t}=\{1,\cdots,t\}\cap\mathcal{T}$, which is a collection of time steps. Let $\{1,2\} \in \mathcal{T}$. Define the learner’s observation history (or interaction sequence) before time $t$ as $\mathcal{H}^{o}_{t} =$ $\left\{(I_{s}, r_{I_{s}})\right\}_{s=1}^{t}$, which represents all information
available for the learner at time $t$.

\begin{figure}[htb]
\centering 
\includegraphics[width=1\textwidth]{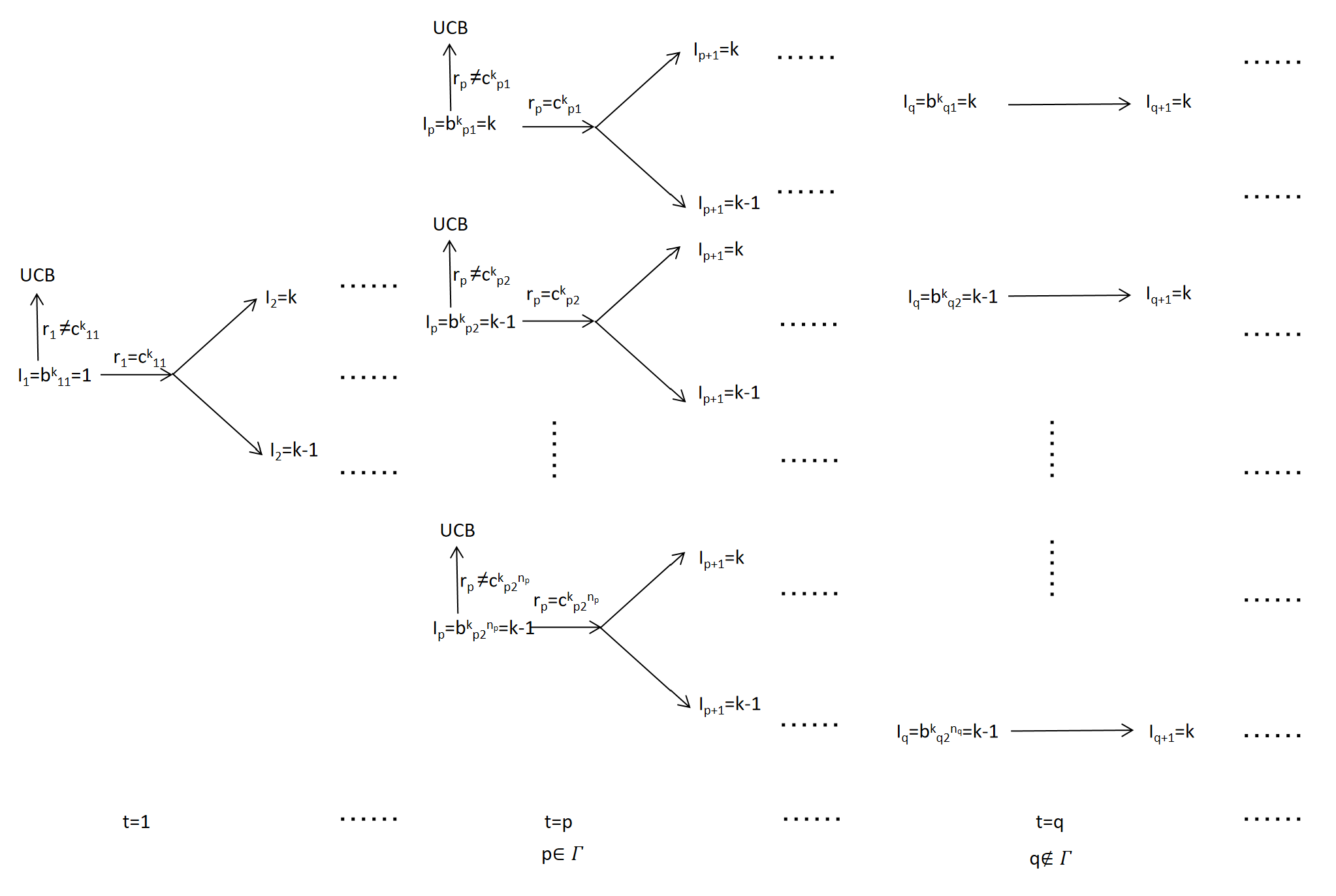}
\caption{Schematic of the algorithm.
\begin{itshape}
We visualize our algorithm. Once the algorithm starts, the learner will choose arm $1$ in the first step. If the interaction sequence up to time $t-1$ $\mathcal{H}^{o}_{t-1}$ is a trigger sequence, the interaction sequence $\mathcal{H}^{o}_{t}$ and whether the time $t$ belongs to $\mathcal{T}$ will determine how the learner will make the next decision. In this case, if  $t=1,p \in \mathcal{T}$ the learner takes a "special action" only if $\mathcal{H}_{t}$ is also a trigger sequence. Otherwise, if we only have $\mathcal{H}^{o}_{t-1}$ as a trigger sequence, the learner will start applying the UCB1 algorithm. If $t=q \notin \mathcal{T}$ the learner will only choose to play arm $k$ no matter what $r_{t}$ is. 
\end{itshape}
}
\end{figure}

Let's start by explaining the basic idea behind our construction. For an effective learning algorithm (e.g., UCB1 or $\epsilon$-greedy, and in this proof we choose UCB1), we add several ``special actions'' to its action set at each time. Each ``special action'' is attached to an arm $k \in [N]$, which belongs to two categories: if the time $t \in \mathcal{T}$, the added special action requires the learner to pull arm $k$ or arm $k-1$ with the same probability $\frac{1}{2}$; and if the time $t \notin \mathcal{T}$, the added special action requires the learner to pull the arm $k$ directly. At each time $t$, the learner will only take one of the special actions based on a pre-specified condition. When there is no adversarial attack, the expectation of the number of times these conditions are satisfied is $o(T)$, which means that in $T-o(T)$ rounds the learner takes action following the original effective algorithm UCB1, and therefore the learning algorithm remains effective after introducing these special actions. At the same time, these conditions can be satisfied by the attacker via modifying the rewards of the pulled arms, which triggers the learner to take a special action related to the target arm $K$ in each round. And at least $T-|\mathcal{T}|$ rounds the learner takes the special action of pulling the target arm $K$ directly. Therefore, when $| \mathcal {T} | = o (T)$, the learner will be manipulated to pull the target arm $T-o(T)$ times.

The key in the construction then becomes how to design conditions that trigger the learner to take special actions. We construct a set $P_{k}$ for $\forall k \in [N]$. When the interaction sequences $\mathcal{H}^{o}_{t}$ consisting of $\{(r_s, I_s)\}_{1 \leq s \leq t}$ at time $t$   belongs to the set $P_{k}$, the learner takes the corresponding special action related to arm $k$ at time $t$. By properly constructing all the $P_{k}, k \in [N]$,  the attacker can modify the reward so as to manipulate the learner's next action. Thus, the attacker can manipulate the learner's interaction sequence $\mathcal{H}^{o}_{t}$ so that $\mathcal{H}^{o}_{t}$ always belongs to the set $P_{K}$, and thus the learner always takes the special actions related to target arm $K$. 

\paragraph{Construction of $P_{k}$.} Next, $\forall k \in [N]$ we give the construction of the set $P_{k}$. We first define the sequence $\{c_{tj}^{k}\} ,\{b_{tj}^{k}\} (k \in [N], t \geq 1,  j \in [2^{t-1}]) \text{ and } \{ n_{t}\}(t \geq 1)$ that we will use later. For sequence $\{c_{tj}^{k}\}$, we set $c_{11}^{k}, \forall k\in [N]$ to $N$ arbitrarily chosen but different real numbers, and when $t>1$ and $j \in [2^{t-1}]$ we set $c_{tj}^{k} = c_{11}^{k}+\frac{\alpha }{j^{2}}$, where $\alpha=\frac{1}{\pi^{2}}\sqrt{72 \sigma^{2} \log \frac{\pi^{2} N}{3}}$. 
For sequence $\{b_{tj}^{k}\}$, we first set $b_{11}^{k}=1$. For 
 $t \geq 2 \text{ and } t-1 \in \mathcal{T}$, we set $b_{tj}^{k} = k~ (\text{if } j \text{ is an odd integer}) \text{ or } k-1~ (\text{if } j \text{ is an even integer})$. When $k$=1, we set $b_{tj}^{k}=N$. When $ t \geq 2 \text{ and } t-1 \notin \mathcal{T}$ we set $b_{tj}^{k} = k$. 
 Finally for sequence $\{ n_{t}\}$ we set $n_{1}=1$ and when $ t \geq 1 \text{ we set } n_{t+1} = n_{t}+1~(\text{if } t \in \mathcal{T}) \text{ or } n_{t}~(\text{if } t \notin \mathcal{T})$. 

To help the reader understand our later constructions intuitively, we use the defined sequences $\{c_{tj}^{k}\}$, $\{b_{tj}^{k}\}$ and $\{ n_{t}\}$ to describe a $T$-level ordered tree related to $k$, where each level of this tree represents a time $t$ and each node at level $t$ represents the combination of the arm chosen by the learner and the reward received $(r_t, I_t)$ (if $t \in \mathcal{T}$) or only the arm chosen by the learner $I_{t}$ (if $t \notin \mathcal{T}$). 
Each node at level $t$ has exactly two sub-nodes when $t \in \mathcal{T}$, otherwise they only have one sub-node. It is easy to see that the number of nodes in the $t$th layer is exactly $2^{n_t}$. For each $t$, we assign the $j$th node of the $t$th layer  value $(r_t,I_t)= (c^{k}_{tj}, b^{k}_{tj})~(\text{if } t \in \mathcal{T}) \text{ or the value } I_t = b^{k}_{tj}(\text{if } t \notin \mathcal{T})$. A path $p$ in this tree consists of nodes $n_1,\cdots,n_s$ is also an interaction sequence containing values in $n_1,\cdots,n_s$. If path $p$ passes node $n = (r, I)$ in the $t$th layer, then for the reward and pulled arm at time $t$ in interaction sequence $p$ we have $(r_t, I_t) =(r, I)$. If path $p$ passes node $n = I$ in the $t$th layer, then for the reward and pulled arm at time $t$ in interaction sequence $p$ we have $(r_t, I_t) = (r, I)$ where $r$ can be any real value. 

%Now we define the ``trigger sequence'', which works with the learner's algorithm to be discussed later to make sure that when the history of pulls $\mathcal{H}_{t}$ matches with a trigger sequence, the learner's algorithm will take the predefined ``special action''. 
Now we can define $P_{k}$ concretely. A sequence of positive integers $\{j_{s}\}_{1 \leq s \leq t}$ is called a \textit{structure sequence}, if $\forall 1 \leq s \leq t-1 \text{ we have } j_{1}=1, j_{s+1}=j_{s}~(\text{if } s \notin \mathcal{T}), \text{ or } j_{s+1} \in \{2j_{s},2j_{s}-1\}~(\text{if } s \in \mathcal{T})$.
An interaction sequence $\mathcal{H}^{o}_{t}$ up to time $t$  is called a $\textit{trigger sequence}$ concerning $k$ at time $t$, if there is an integer $k$ and a structure sequence $\{j_{s}\}_{1 \leq s \leq t}$ such that for $\forall 1 \leq s \leq t$ the sequences $\{r_s\}$ and $\{I_{s}\}$ in $\mathcal{H}^{o}_{t}$ satisfy $r_{s}=c^{k}_{sj_{s}}(\text{if } s \in \mathcal{T}) \text{ and } I_{s} = b_{sj_{s}}^{k}$. We call the structure sequence $\{j_{s}\}_{1 \leq s \leq t}$ the \textit{structure} of the trigger sequence $\mathcal{H}^{o}_{t}$. 
The trigger sequence we define here is a path in the $T$-level tree from the first level to the $t$th level, and its structure $\{j_{s}\}_{1 \leq s \leq t}$ is a sequence of the positions of the nodes at each level of the path. Obviously, if $\mathcal{H}^{o}_{t}$ is a trigger sequence then $\mathcal{H}_{t-1}$ is also a trigger sequence. Finally, we define $P_{k}$ to be the set of all trigger sequences concerning arm $k$.

\paragraph{The learning algorithm.} Now we are equipped to present the learner's algorithm. For $t=1$, the learner will choose arm $1$ and get the reward $r_1$. For $t = 2$, if there is a $k$ such that $r_{1}=c_{11}^{k}$, then 
the learner will pull arm $k$ or arm $k-1~(\text{if } k=1, \text{it's arm } N)$  with probability $\frac{1}{2}$. If there is not a satisfying $k$, the learner resets the history information it has so far (including $r_1$ and $I_1$) and starts to apply the UCB1 algorithm after time $t=1$. 

For $t>1$ and $t\in \mathcal{T}$, suppose there is an integer $k$ and the interaction sequence $\mathcal{H}^{o}_{t-1}$ at time $t-1$  is a trigger sequence concerning $k$ and $t-1$. If the interaction sequence $\mathcal{H}^{o}_{t}$ at time $t$ is also a trigger sequence concerning $k$ and $t$, $i.e. \mathcal{H}^{o}_{t}\in P_{k}$, then the learner will pull arm $k$ or arm $k-1(\text{if } k=1, \text{it's arm } N)$  with probability $\frac{1}{2}$. In any case, if we set $j_{t + 1} = 2 j_{t}-1 ~(\text{if } I_{t + 1} = k)$ or $2 j_{t} ~(\text{if } I_{t + 1} = k - 1) $, then we can see that $I_{t + 1} = b^{k}_{(t+1)j_{t + 1}} $, and $\mathcal{H}^{o}_{t}\cup\{(c^{k}_{(t+1)j_{t+1}},b^{k}_{(t+1)j_{t + 1}})\}$ is also a trigger sequence concerning $k$ and $t+1$. But if $\mathcal{H}^{o}_{t}$ is not a trigger sequence concerning $k$ and $t$, then the learner ignores all the history information it has so far and starts to apply the UCB1 algorithm after time $t$. 
That is, in this case, the learner takes a ``special action'' and gets corresponding $I_{t+1}$ only if $\mathcal{H}^{o}_{t}$ is a trigger sequence. And as a further consequence of taking this ``special action'', $\exists j_{t+1}$, once $r_{t+1} = c^{k}_{(t+1)j_{t+1}}$, $\mathcal{H}^{o}_{t+1} = \mathcal{H}^{o}_{t}\cup\{(r_{t+1}, I_{t+1})\}$ is also a trigger sequence $\in P_{k}$. Otherwise, if we only have $\mathcal{H}^{o}_{t-1}$ as a trigger sequence, the learner will start applying the UCB1 algorithm. 

For $t>1$ and $t \notin \mathcal{T}$, suppose there is an integer $k$ and the interaction sequence $\mathcal{H}^{o}_{t-1}$ at time $t-1$ is a trigger sequence concerning $k$ and $t-1$. In this case, the learner will choose arm $k$ at time $t+1$. Further, if we set $j_{t+1} = j_{t}$ then we have $I_{t+1}=k=b^{k}_{(t+1)j_{t+1}}$, and $\mathcal{H}^{o}_{t}\cup\{(r_{t+1},b^{k}_{(t+1)j_{t + 1}})\}$ is also a trigger sequence concerning $k$ and $t+1$ for any reward $r_{t+1}$. That is, in this case, the learner takes a ``special action'' and gets corresponding $I_{t+1}$ if $\mathcal{H}^{o}_{t-1}$ is a trigger sequence. And as a further consequence of taking this ``special action'', $\mathcal{H}^{o}_{t+1} = \mathcal{H}^{o}_{t}\cup\{(r_{t+1},I_{t+1})\}$ is also a trigger sequence $\in P_{k}$ for any reward $r_{t+1}$. 

We can prove that our described algorithm is well-defined, i.e., the algorithm has a valid output under all circumstances it might encounter. For any time $t>1$, given any integer $k$ there is a maximum time $t^{'} \leq t$ (if such $t^{'}>1$ does not exist we set it to $1$) such that the interaction sequence $\mathcal{H}^{o}_{t^{'}}$ up to time $t^{'}$ is a trigger sequence concerning $k$ at time $t^{'}$ and its structure is $\{j_{s}\}_{1 \leq s \leq t^{'}}$. If $\forall 1 \leq k \leq N$ the corresponding $t^{'}=1$, from the learning algorithm we know the learner will ignore the previous information and start to apply UCB1 after time $1$. 
Hence, the algorithm has a valid output at time $t$. In other cases the corresponding to $t^{'}>1$, $\exists k$ such that $r_{1}=c^{k}_{11}$ holds and there is only one integer $k$ making its corresponding $t^{'}$ is not $1$. If its corresponding $t^{'} = t$ or $t-1$, which means the interaction sequence $\mathcal{H}^{o}_{t}$ at time $t$ is a trigger sequence concerning $k$ and $t$, then the algorithm still has a valid output at time $t$. 
Now suppose its corresponding $t^{'} < t-1$. It is easy to see that $t^{'} \in \mathcal{T}$, otherwise in time $t^{'}+1$ the learner will choose arm $k$ and if we define $j_{t^{'}+1} = j_{t^{'}}$ we have $I_{t^{'}+1}=k=b^{k}_{(t^{'}+1)j_{t^{'}+1}}$. Therefore, $\mathcal{H}^{o}_{t^{'}+1}$ is also a trigger sequence which contradicts the definition of $t^{'}$. Therefore, it must be true that $t^{'} \in \mathcal{T}$ and $\mathcal{H}^{o}_{t^{'}+1}$ is not a trigger sequence, according to the learning algorithm we know the learner will ignore the previous information and play UCB1 after time $t^{'}$. Again, the algorithm has a valid output at time $t$. In summary, we have proved that the algorithm we construct is well-defined.

\paragraph{The attack method.} Now, given any target arm $K$, we describe the attack method against our proposed general bandit algorithm. For $t=1$, the attacker replaces $r_{1}^{0}$ by $r_{1}=c_{11}^{K}$. For $t>1$ and $t \notin \mathcal{T}$, the attacker will not attack.

For $t>1$ and $t \in \mathcal{T}$, suppose there is an integer $K$ and the interaction sequence $\mathcal{H}^{o}_{t-1}$ at time $t-1$ is a trigger sequence concerning $K$ at time $t$, and $\{j_{s}\}_{1 \leq s \leq t-1}$ is its structure. From the learning algorithm, we know that the learner will take "special action" in time $t$ and get the pulled arm $I_{t}$. As a consequence $\exists j_{t}$ such that $\mathcal{H}^{o}_{t} = \mathcal{H}^{o}_{t-1}\cup\{(c^{k}_{tj_{t}},I_{t})\}$ is also a trigger sequence concerning $k$ and $t$. Then the attacker replaces $r_{t}^{0}$ by $r_{t}=c^{K}_{tj_{t}}$. But if such $K$ or trigger sequence $\mathcal{H}^{o}_{t-1}$ does not exist, the attacker will not attack at time $t$. In other words, the attacker only attacks when time $t \in \mathcal{T}$ and $\mathcal{H}^{o}_{t-1}$ is a trigger sequence, and will change the reward at time $t$ to make $\mathcal{H}^{o}_{t}$ also a trigger sequence. 

Thus, for $t \in \mathcal{T}$, once the condition for triggering the "special action" at time $t$ is satisfied (that is, $\mathcal{H}^{o}_{t-1} \in P_{k}$), the attacker can then modify the reward value of the arm pulled in round $t$ such that the condition for triggering the "special action" at time $t+1$ is satisfied (that is, $\mathcal{H}^{o}_{t} \in P_{k}$). At the same time, when $t \notin \mathcal{T}$and $\mathcal{H}^{o}_{t-1} \in P_{k}$, we know by the learning algorithm that there must be $\mathcal{H}^{o}_{t} \in P_{k}$ at this time. Then as long as the attack method is followed from the first round, the learner will always take the "special action". This makes our attack effective as well. We will give a more formal proof of this later.

Now, let's set $\mathcal{T}=\{2^{s}, s \geq 0\}$ then $|\mathcal{T}_{t}| \leq 1+\log_{2} t$, which fills the final blanks in the learner's algorithm and the attacker's method. Given any environment $(F_{1},\cdots,F_{N})$, we prove that for any sub-optimal arm $i$ we have $E(N_{i}(T)) = o(T)$. 

\paragraph{The learning algorithm is effective.} We first prove that the learning algorithm is effective. For any $t$, we will bound the probability that the learner takes exactly $t$ special actions. Without loss of generality, we assume that the target arm is arm $K$. According to our algorithm, for any given interaction sequence $\mathcal{H}^{o}_{T}$ at time $T$ we can find a time after which the learner chooses to apply the UCB1 algorithm (if the learner does not apply the UCB1 algorithm the entire time, then it's $T$) and we use $t_{0}$ to denote that time for $\mathcal{H}^{o}_{T}$. Consider any given interaction sequence  $\mathcal{H}^{o}_{T}$, because of the conditions under which we take the ``special action'' in the algorithm 
we know that $t_{0} \in \mathcal{T}$, $\mathcal{H}^{o}_{t_{0}-1}$ is a trigger sequence but $\mathcal{H}^{o}_{t_{0}}$ is not a trigger sequence. In other words, there is a structure sequence $\{j_{s}\}_{1 \leq s \leq t_{0}}$ such that $\forall 1\leq t \leq t_{0}-1$ we have $r_{t}=c^{K}_{tj_{t}}$ (if $t \in \mathcal{T}$) and $I_{t}=b^{K}_{tj_{t}}$, and we also have $I_{t_{0}}=b^{K}_{t_{0}j_{t_{0}}}$ and $\mathcal{H}^{o}_{t_{0}-1}\cup\{(c^{K}_{t_{0}j_{t_{0}}}, b^{K}_{t_{0}j_{t_{0}}})\}$ is a trigger sequence but $r_{t_{0}} \neq c^{K}_{t_{0}j_{t_{0}}}$. We use $\mathcal{B}(\mathcal{H}^{o}_{T})$ to denote the structure sequence $\{j_{s}\}_{1 \leq s \leq t_{0}}$. 

According to our algorithm design, it is not hard to verify that 
$$
\begin{aligned}
\mathbb{P}(\{ \mathcal{B}(\mathcal{H}^{o}_{T})=\{j_{s}\}_{1 \leq s \leq t_{0}} \}) &= \mathbb{P}(\{I_{1}=b^{K}_{1j_{1}},\cdots ;I_{t_{0}} = b^{K}_{t_{0}j_{t_{0}}};r_{t_{0}} \neq c^{K}_{t_{0}j_{t_{0}}} \text{ and } \forall s \in \mathcal{T}_{t_{0}-1}\enspace r_{s}=c^{K}_{sj_{s}}\}) \\
&=(\frac{1}{2})^{|\mathcal{T}_{t_{0}}|-1}\mathbb{P}(r_{t_{0}} \neq c^{K}_{t_{0}j_{t_{0}}} \mid I_{t_{0}}=b^{K}_{t_{0}j_{t_{0}}}) \prod_{s \in \mathcal{T}_{t_{0}-1}}\mathbb{P}( r_{s} = c^{K}_{sj_{s}} \mid I_{s}=b^{K}_{sj_{s}}) \\
&\leq (\frac{1}{2})^{|\mathcal{T}_{t_{0}}|-1} \mathbb{P}( r_{\frac{t_{0}}{2}} = c^{K}_{\frac{t_{0}}{2}j_{\frac{t_{0}}{2}}} \mid I_{\frac{t_{0}}{2}}=b^{K}_{\frac{t_{0}}{2} j_{\frac{t_{0}}{2}}})
\end{aligned}
$$
We use $N^{'}_{i}(T)$ to denote the number of times arm $i$ is pulled up to round $T$, when the learner chooses to apply the UCB1 procedure at the beginning of the game. Then for $\mathcal{H}^{o}_{T}$ such that $\mathcal{B}(\mathcal{H}^{o}_{T})=\{j_{s}\}_{1 \leq s \leq t_{0}}$ we have 
$$
N_{i}(T) \leq t_{0}+N^{'}_{i}(T-t_{0}) \leq t_{0}+N^{'}_{i}(T)
$$
Also, notice that for $t>2$ and $t \in \mathcal{T}$
$$
\begin{aligned}
\mathbb{P}(\{t_{0} = t\}) \leq& \sum_{\{j_{s}\}_{1 \leq s \leq t} \text{is a structure sequence}}\mathbb{P}(\{\mathcal{B}(\mathcal{H}^{o}_{T})=\{j_{s}\}_{1 \leq s \leq t}\}) \\
\leq& (\frac{1}{2})^{|\mathcal{T}_{t}|-1}\sum_{\{j_{s}\}_{1 \leq s \leq t}\text{is a structure sequence}}\mathbb{P}( r_{\frac{t}{2}} = c^{K}_{\frac{t}{2}j_{\frac{t}{2}}} \mid I_{\frac{t}{2}}=b^{K}_{\frac{t}{2} j_{\frac{t}{2}}})\\
\leq& (\frac{1}{2})^{|\mathcal{T}_{t}|-1}\sum_{\{j_{s}\}_{1 \leq s \leq t}\text{is a structure sequence}, b^{K}_{\frac{t}{2} j_{\frac{t}{2}}} = K-1}\mathbb{P}( r_{\frac{t}{2}} = c^{K}_{\frac{t}{2}j_{\frac{t}{2}}} \mid I_{\frac{t}{2}}=K-1)  \\&+(\frac{1}{2})^{|\mathcal{T}_{t}|-1}\sum_{\{j_{s}\}_{1 \leq s \leq t}\text{is a structure sequence}, b^{K}_{\frac{t}{2} j_{\frac{t}{2}}}=K}\mathbb{P}( r_{\frac{t}{2}} = c^{K}_{\frac{t}{2}j_{\frac{t}{2}}} \mid I_{\frac{t}{2}}=K) \\
\leq& (\frac{1}{2})^{|\mathcal{T}_{t}|-2}\sum_{\{j_{s}\}_{1 \leq s \leq \frac{t}{2}}\text{is a structure sequence}, b^{K}_{\frac{t}{2} j_{\frac{t}{2}}} = K-1}\mathbb{P}( r_{\frac{t}{2}} = c^{K}_{\frac{t}{2}j_{\frac{t}{2}}} \mid I_{\frac{t}{2}}=K-1)  \\&+(\frac{1}{2})^{|\mathcal{T}_{t}|-2}\sum_{\{j_{s}\}_{1 \leq s \leq \frac{t}{2}}\text{is a structure sequence}, b^{K}_{\frac{t}{2} j_{\frac{t}{2}}}=K}\mathbb{P}( r_{\frac{t}{2}} = c^{K}_{\frac{t}{2}j_{\frac{t}{2}}} \mid I_{\frac{t}{2}}=K) \\
\end{aligned}
$$
Notice that when $t>1$ for any structure sequence $\{j_{s}\}_{1 \leq s \leq t} \neq \{j^{'}_{s}\}_{1 \leq s \leq t}$ we have $j_{t} \neq j^{'}_{t}$. Therefore we have 
$$
\begin{aligned}
&(\frac{1}{2})^{|\mathcal{T}_{t}|-2}\sum_{\{j_{s}\}_{1 \leq s \leq \frac{t}{2}}\text{is a structure sequence}, b^{K}_{\frac{t}{2} j_{\frac{t}{2}}} = K-1}\mathbb{P}( r_{\frac{t}{2}} = c^{K}_{\frac{t}{2}j_{\frac{t}{2}}} \mid I_{\frac{t}{2}}=K-1)  \\
&+(\frac{1}{2})^{|\mathcal{T}_{t}|-2}\sum_{\{j_{s}\}_{1 \leq s \leq \frac{t}{2}}\text{is a structure sequence}, b^{K}_{\frac{t}{2} j_{\frac{t}{2}}}=K}\mathbb{P}( r_{\frac{t}{2}} = c^{K}_{\frac{t}{2}j_{\frac{t}{2}}} \mid I_{\frac{t}{2}}=K) \\
= &
(\frac{1}{2})^{|\mathcal{T}_{t}|-2}\mathbb{P}(\bigcup_{\{j_{s}\}_{1 \leq s \leq \frac{t}{2}}\text{is a structure sequence}, b^{K}_{\frac{t}{2} j_{\frac{t}{2}}} = K-1}\{r_{\frac{t}{2}} = c^{K}_{\frac{t}{2}j_{\frac{t}{2}}} \}\mid I_{\frac{t}{2}}=K-1)  \\
&+(\frac{1}{2})^{|\mathcal{T}_{t}|-2}\mathbb{P}(\bigcup_{\{j_{s}\}_{1 \leq s \leq \frac{t}{2}}\text{is a structure sequence}, b^{K}_{\frac{t}{2} j_{\frac{t}{2}}} = K-1}\{r_{\frac{t}{2}} = c^{K}_{\frac{t}{2}j_{\frac{t}{2}}} \}\mid I_{\frac{t}{2}}=K) \\
\leq& (\frac{1}{2})^{|\mathcal{T}_{t}|-2} + (\frac{1}{2})^{|\mathcal{T}_{t}|-2} \\
=& (\frac{1}{2})^{|\mathcal{T}_{t}|-3}\\
\leq& (\frac{1}{2})^{\log_{2} t -2} = \frac{4}{t}
\end{aligned}
$$

As a result, for $\forall t > 2$ and $t \in \mathcal{T}$,
$$
\begin{aligned}
\mathbb{P}(\{t_{0} = t\}) \leq \frac{4}{t}
\end{aligned}
$$
Then for any sub-optimal arm $i$ we can bound $E(N_{i}(T))$ as follows:
$$
\begin{aligned}
E(N_{i}(T)) &=\sum_{t \in \mathcal{T}_{T}}\mathbb{P}(\{t_{0} = t\})E(N_{i}(T) | \{t_{0} = t\}) \\
&\leq \sum_{t \in \mathcal{T}_{T}, t>2}(t+E(N^{'}_{i}(T)))\frac{4}{t} + 3 + 2E(N^{'}_{i}(T))\\
&\leq 6E(N^{'}_{i}(T))+4|\mathcal{T}_{T}| + 3\\
&\leq 6E(N^{'}_{i}(T))+4(1+\log_{2} T) +3\\
&= \Theta(\log T)
\end{aligned}
$$

Finally, given any environment $(F_{1},\cdots, F_{N})$, for any target arm $K \in \{1,\cdots, N\}$, we prove that with probability at least $1-\eta$, the attacker can force the learner choose the target arm for $T-O(\log T)$ times with a cumulative attack cost at most $O(\log T \sqrt{\log (\frac{1}{\eta})})$. And we prove the attack method is stealthy.

\paragraph{The attack method will always succeed.} Now we give formal proof to show that under the attack, the learner will always take "special action". Preliminary explanations have been mentioned at the end of the section on constructing the attack methods. Suppose the target arm is arm $K$. By induction, we only need to prove that $\forall t>1$ there exists a trigger sequence concerning $K$ at time $t$. 

First of all, because $r_{1}=r'=c^{K}_{11}$, the learner will choose to play arm $K$ or arm $K-1$ with equal probability at time $2$. Let $k=K$ and $\{j_{1},j_{2}\}=\{1,1\}~(\text{if } I_{2}=K) \text{ or } \{1,2\}~(\text{if } I_{2}=K-1)$, we have $j_{1}=1,j_{2}\in\{2j_{1},2j_{1}-1\},r_{1}=c^{K}_{11},I_{1}=1=b^{K}_{1j_{1}}$ and $I_{2}=b^{K}_{2j_{2}}$, so $\{j_{s}\}_{1 \leq s \leq 2}$ is the trigger sequence concerning $K$ at time $t=2$. So the proposition holds for $t=2$.

Now suppose it holds for $t=\Tilde{t}$. Then we can find the trigger sequence $\mathcal{H}^{o}_{\Tilde{t}}$ concerning $K$ at time $\Tilde{t}$ and its structure is
$\{j_{s}\}_{1 \leq s \leq \Tilde{t}}$. If $\Tilde{t} \notin \mathcal{T}$, the learner will play arm $K$. Let $j_{\Tilde{t}+1}=j_{\Tilde{t}}$, then we have $I_{\Tilde{t}+1}=K=b^{K}_{(\Tilde{t}+1)j_{\Tilde{t}+1}}$. So $\mathcal{H}^{o}_{t}\cup\{b^{K}_{(\Tilde{t}+1)j_{\Tilde{t} + 1}}\}$ is also a trigger sequence concerning $K$ at time $\Tilde{t}+1$. If $\Tilde{t} \in \mathcal{T}$, from the attack method we know $r_{\Tilde{t}}=r^{'}_{\Tilde{t}}=c_{\Tilde{t}}j_{\Tilde{t}}$. So from the algorithm, we know that the learner will choose to play arm $K$ or arm $K-1$ with equal probability. If we define $j_{\Tilde{t} + 1} = 2 j_{\Tilde{t}}-1 (\text{if } I_{\Tilde{t} + 1} = K)$ or $2 j_{\Tilde{t}}  (\text{if } I_{\Tilde{t} + 1} = K - 1) $, it means $I_{\Tilde{t} + 1} = b^{K}_{(\Tilde{t}+1)j_{\Tilde{t} + 1}} $, and $\mathcal{H}^{o}_{\Tilde{t}}\cup\{(c^{K}_{(\Tilde{t}+1)j_{\Tilde{t}+1}},b^{K}_{(\Tilde{t}+1)j_{\Tilde{t} + 1}})\}$ is also a trigger sequence concerning $K$ and $\Tilde{t}+1$. Therefore the proposition holds for $t=\Tilde{t}+1$.

Consequently, by the principle of mathematical induction, the proposition holds. Hence, the learner will choose to play arm $K$ when $t-1 \notin \mathcal{T}$, which lead to $N_{K}(T) \geq T-|\mathcal{T}_{T}| \geq T-O(\log T)$. 

\paragraph{The attack method is stealthy.} Next, for $\eta \in (0,1)$ suppose $E_{\eta}$ we defined in Section \ref{section 4} holds. From the above sequence and attack method, we know that for any time $t$ and any non-target arm $i$ if $I_{t} = i$ then $\exists j_{t} \in \{1,\cdots,2^{t-1}\}$ such that the post-attack reward $r_{t} = c^{k}_{tj_{t}}$. Then $\forall 1 \leq t_{1} < t_{2} \leq T $, from the definition of $c^{k}_{tj_{t}}$ we know that  
$$
\begin{aligned}
|\hat{\mu}_{i}(t_{1})-\hat{\mu}_{i}(t_{2})|&=\left|N_{i}(t_{1})^{-1} \sum_{\left\{s: s \le t_1, I_{s}=i\right\}} \frac{\alpha}{s^{2}}-N_{i}(t_{2})^{-1} \sum_{\left\{s: s \le t_2, I_{s}=i\right\}} \frac{\alpha}{s^{2}}\right|\\
&\leq \alpha(N_{i}\left(t_{1})^{-1}\sum_{s=1}^{N_{i}(t_{1})} \frac{1}{s^{2}}+N_{i}(t_{2})^{-1}\sum_{s=1}^{N_{i}(t_{2})} \frac{1}{s^{2}}\right)\\
&< \frac{\alpha \pi^{2}}{6}(N_{i}(t_{1})^{-1}+N_{i}(t_{2})^{-1})\\
&\leq \beta(N_{i}(t_{1}))+\beta(N_{i}(t_{2}))
\end{aligned}
$$
The last inequality holds because $\forall n \geq 1$ we have $\frac{\alpha \pi^{2}}{6}=\sqrt{2 \sigma^{2} \log \frac{\pi^{2} N }{3}} < \sqrt{2 \sigma^{2} \log \frac{\pi^{2} N }{3 \delta}}\leq \sqrt{2 \sigma^{2}n \log \frac{\pi^{2} N n^{2}}{3 \delta}} = n \beta(n)$. 
As a result, the attack will not be detected in this case. For target arm $K$, there will be no attack. From the above results, we have proved that the attack will not be detected with a probability higher than $\eta$.

Finally, suppose $E_{\eta}$ holds, because the attacker only attacks when $t \in \mathcal{T}$, the cumulative attack cost is no more than $(\max_{i,k}|\mu_{i}-c^{k}_{11}|+\beta(1)+\beta(1,\eta))|\mathcal{T}_{T}|=O(\log T \sqrt{\log (\frac{1}{\eta})})$, with probability at least $1-\eta$.
\end{proof}

\subsection{Proof of Theorem \ref{theorem: EARR}}
\begin{proof}
Given any $\eta \in (0,1)$ and environment $(F_{1},\cdots,F_{N})$, we set $M(T)=\max \left\{  \frac{2 C_{K}}{\Delta^{0}_{1K}-2\beta(1)} , \frac{16 \sigma^{2} \ln T}{(\Delta^{0}_{1K}-2\beta(1))^{2}}\right\} $.
Notice that because the attack in the history of pulls is not detected, $\forall 1 \leq j \leq N,  \bigcap_{t=1}^{T}(\hat{\mu}_{j}(N_{j}(t))-\beta\left(N_{j}(t)\right),\hat{\mu}_{j}(N_{j}(t))+\beta\left(N_{j}(t)\right)) \neq \emptyset$, therefore $\forall 1 \leq j \leq N$ we can find a $\mu_{j}^{'} \in \bigcap_{t=1}^{T}(\hat{\mu}_{j}(N_{j}(t))-\beta\left(N_{j}(t)\right),\hat{\mu}_{j}(N_{j}(t))+\beta\left(N_{j}(t)\right))$.  
And therefore
$$
\begin{aligned}
&\mathbb{P}(\mu^{'}_{K} > \mu^{'}_{1}-\nu \Delta | N_{K}(T)\geq M(T) ) \\  
\leq &
\mathbb{P}(\hat{\mu}_{K}(T)+\beta(N_{K}(T)) > \hat{\mu}_{1}(N)-\beta(1)-\nu \Delta | N_{K}(T)\geq M(T) ) \\
\leq & \mathbb{P}(\hat{\mu}^{0}_{K}(T)+\beta(1)+\frac{\Delta}{2}>\hat{\mu}_{1}(N)-\beta(1)-\nu \Delta | N_{K}(T) \geq M(T)) \\
= &\mathbb{P}(\hat{\mu}^{0}_{K}(T)-\mu_{K} > (\frac{1}{2}-\nu)\Delta_{g} | N_{K}(T)\geq M(T)) \\
\leq &\frac{1}{T^{8(\frac{1}{2}-\nu)^2}}
\end{aligned}
$$
The last inequality is based on Lemma \ref{lemma: hoeffding} and the second last inequality is because $\hat{\mu}_{K}(t-1)-\hat{\mu}_{K}^{0}(t-1) = \frac{C_{K}}{N_{K}(t-1)} = \frac{C_{K}}{M(T)} \leq \frac{\Delta}{2} $. 
For $\forall 1 \leq j \leq N$, we can regard $\mu_{j}^{'}$ as the mean reward of arm $j$ and in this case, there is no attack on the reward (as the attacker does not attack for the first pull of each arm). Therefore, if $N_{K}(T)\geq M(T)$ we know that  $\mathbb{P}(\mu^{'}_{i}-\nu \Delta \geq \mu^{'}_{K}) \geq 1-\frac{1}{T^{8(\frac{1}{2}-\nu)^2}}$. \\
Notice that $\forall j \text{ and } 1 \leq t \leq T,  \hat{\mu}_{j}(t) \in (\mu_{j}^{'}-\beta\left(N_{j}(t)\right),\mu_{j}^{'}+\beta\left(N_{j}(t)\right)) \subset (\mu_{j}^{'}-l\left(N_{j}(t)\right),\mu_{j}^{'}+l\left(N_{j}(t)\right)) $. Now by applying the definition of the EARR algorithm, we can find a function $g(t,\eta): \frac{g(t,\eta)}{t} \to 0  \text{ as } t \to  \infty$ such that with probability at least $1-\eta$ we have 
$$ N_{K}(T)\max_{1\leq j \leq N}\{\mu_{j}^{'}-\mu_{K}^{'}\} \leq R_{T} \leq g(T,\eta)
$$ 
Because $\max_{1\leq j \leq N}\{\mu_{j}^{'}-\mu_{K}^{'}\}  > \mu^{'}_{1}-\mu^{'}_{K} \geq  \nu \Delta $, we have
$$
N_{K}(T) \leq \frac{g(T,\eta)}{\nu\Delta}
$$
In summary, we finally conclude that with probability at least $1-\eta-\frac{1}{T^{8(\frac{1}{2}-\nu)^2}}$
$$
N_{K}(T) \leq \max \{ \dfrac{g(T,\eta)}{\nu  \Delta},\frac{16 \sigma^{2} \ln T}{\Delta^{2}},\frac{2 C(T)}{\Delta}\}
$$
Then we complete the proof.
\end{proof}

\subsection{Proof of Lemma \ref{lemma: ERR}}
\begin{proof}
We make two adjustments to the algorithm given in the proof of Lemma \ref{lemma: construction}. 

For $t=1$, if there does not exist a $k$ such that $r_{1}=c_{11}^{k}$, then the learner \textbf{keeps} the history information it has so far and chooses to apply the UCB1 algorithm after time $t$.

For $t>1$ and $t\in \mathcal{T}$, suppose there is a integer $k$ and the trigger sequence with concerning $k$ at time $t$ which has structure $\{j_{s}\}_{1 \leq s \leq t}$. If $r_{t} \neq c^{k}_{tj_{t}}$, the learner \textbf{keeps} the history information it has so far and chooses to apply the UCB1 algorithm after time $t$.  

The rest of the algorithm and the attack method are the same as that given in the proof of Lemma \ref{lemma: construction}. Similarly to the proof of Lemma \ref{lemma: construction}, it is easy to verify that 1) our algorithm is still well-defined, in other words, the learning algorithm can always tell the learner what action to take, 2) the attack method is stealthy, and 3) given any environment $(F_{1},\cdots, F_{N})$ for any target arm $K \in \{1,\cdots, N\}$ with probability at least $1-\eta$, the attacker can force the learner to choose the target arm for $T-O(\log T)$ times with a cumulative attack cost at most $O(\log T \sqrt{\log (\frac{1}{\eta})})$.

Now we only need to prove that this algorithm is an ERR algorithm. Given any $\eta \in (0, 1)$ and environment $(F_{1},\cdots,F_{N})$, define $\mu_{i_{*}} = \max_{i} \mu_{i}$ and $\Delta_{i} = \mu_{i_{*}} - \mu_{i}$. For any arm $j$, if $t>(\frac{\pi^{2} N}{3\eta})^{\frac{2}{5}}$ we have $ 3\sigma\sqrt{\frac{\log t}{N_{j}(t)}}>\beta(N_{j}(t),\eta) $. Consider an interaction sequence $\mathcal{H}^{o}_{T}$ such that $\forall i, \forall N \leq t \leq T,  \hat{\mu}_{i}(t) \in (\mu_{i}-\beta\left(N_{i}(t),\eta\right),\mu_{i}+\beta\left(N_{i}(t),\eta\right))
$ and $\mathcal{B}(\mathcal{H}_{T})=\{j_{s}\}_{1 \leq s \leq t_{0}}$. For any sub-optimal arm $i$ suppose $t>\max\{t_{0},(\frac{\pi^{2} N}{3\eta})^{\frac{2}{5}}\}$ and $N_{i}(t-1) \geq \frac{36\sigma^{2}\log T}{\Delta^{2}_{i}}$ we have the following
$$
\begin{aligned}
\hat{\mu}_{i_{*}}(t-1)+3\sigma\sqrt{\frac{\log t}{N_{i_{*}}(t-1)}} 
&> \mu_{i_{*}}-\beta(N_{i_{*}}(t-1),\eta)+3\sigma\sqrt{\frac{\log t}{N_{i_{*}}(t-1)}} \\
&> \mu_{i_{*}}\\
&> \mu_{i}+6\sigma\sqrt{\frac{\log t}{N_{i}(t-1)}} \\
&> \mu_{i}+\beta(N_{i}(t-1),\eta)+3\sigma\sqrt{\frac{\log t}{N_{i}(t-1)}} \\
&> \hat{\mu}_{i}(t-1)+ 3\sigma\sqrt{\frac{\log t}{N_{i}(t-1)}}
\end{aligned}
$$
This means the learner will not play arm $i$ in this case. Therefore we have $N_{i}(T) \leq \max\{\frac{36\sigma^{2}\log T}{\Delta^{2}_{i}},(\frac{\pi^{2} N}{3\eta})^{\frac{2}{5}},t_{0}\} + 1$. 

We further notice that
$$
\begin{aligned}
\mathbb{P}(\{t_{0}>\log T\}) &= \sum_{t \in \mathcal{T}_{T}, t > \log T}\mathbb{P}(\{t_{0} = t\}) \\
&\leq \sum_{t \in \mathcal{T}_{T}, t > \log T}\frac{1}{2t}\\
&= \sum_{ \log_{2}\log T < s \leq \log_{2}T }\frac{1}{2^{s+1}}\\
&\leq \frac{1}{2^{\log_{2}\log T}} 
=\frac{1}{\log T}
\end{aligned}
$$
Therefore, we have $\mathbb{P}(\{t_{0} \leq \log T\}) > 1-\frac{1}{\log T} > 1-\eta$ for $T>e^{\frac{1}{\eta}}$. If $t_{0} \leq \log T$ we have $R_{T} = \sum_{i}^{N}N_{i}(T)\Delta_{i} \leq \sum_{i=1}^{N}(\max\{\frac{36\sigma^{2}\log T}{\Delta_{i}},(\frac{\pi^{2} N }{3\eta})^{\frac{2}{5}}\Delta_{i},\Delta_{i}\log T\}+1)$. Set $l(n,\eta) = \beta(n,\eta)$, $g(t,\eta)=\sum_{i=1}^{N}\max\{\frac{36\sigma^{2}\log t}{\Delta_{i}},(\frac{\pi^{2} N }{3\eta})^{\frac{2}{5}}\Delta_{i},\Delta_{i}\log t\} + N$.
This is because $\eta$ and $\Delta_{i}$ is given and we know $g(t, \eta)$ is well-defined. Then for $T>e^{\frac{1}{\eta}}$ with probability at least $1-\eta$ we have $R_{T} \leq g(T,\eta)$. Hence, this algorithm is an ERR algorithm.
\end{proof}

\section{Supplementary Experimental Results}\label{appendix C}

\begin{figure*}[htbp]
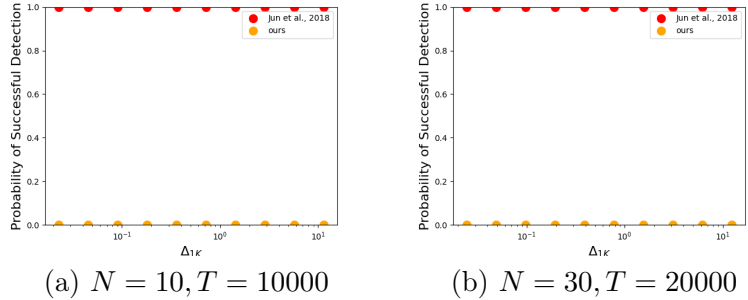

    \centering
    \begin{tabular}{c c}
     \includegraphics[width=5cm]{figures/Probability_of_Successful_Detection_epsilon_1.png} & 
     \includegraphics[width=5cm]{figures/Probability_of_Successful_Detection_epsilon_3.png}\\
    (a) $N = 10, T = 10000$ &
    (b) $N = 30, T = 20000$
    % \vspace{-2mm}
    \end{tabular}
    \caption{Probability of successful detection under \cite{jun2018adversarial}'s attack method when $\epsilon$-greedy is the victim algorithm.}
    \label{fig:p-detect-egreedy}
\end{figure*}

 \begin{figure*}[htbp]
    \centering
    \begin{tabular}{c c}
     \includegraphics[width=5cm]{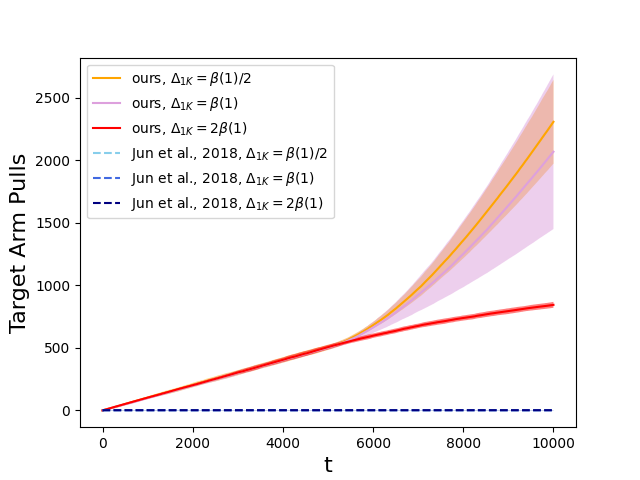} &
     \includegraphics[width=5cm]{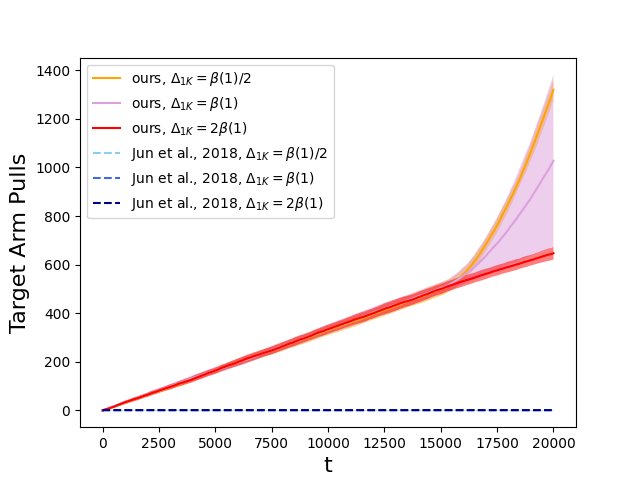} \\
    (a) $N = 10, T = 10000$. &
    (b) $N = 30, T = 20000$.
    % \vspace{-2mm}
    \end{tabular}
    \caption{Target arm pulls under different attack methods when $\epsilon$-geedy is the victim algorithm.}
   
\label{fig:armpull-egreedy}
\end{figure*}

In this section, we present supplementary experimental results. 

Figure \ref{fig:p-detect-egreedy} and Figure \ref{fig:armpull-egreedy} show the experiment results for $\epsilon$-greedy which have been discussed in Section \ref{sec:exp}. 

The results shown in Figure \ref{fig:armpull givenreward-UCB1} and \ref{fig:armpull givenreward-egreedy} more clearly demonstrate the influence of random variable $\hat{\mu}_{1}(N)$ on the environment's attackability for UCB1 and $\epsilon$-greedy. Here we consider the case where $\Delta^{0}_{1K} = \hat{\mu}_{1}(N)-\mu_{K}$ deviates from $\beta(1)$ by $v$: $(1-v)\beta(1)$ vs., $(1+v)\beta(1)$. In our experiments, we set $v=0.2$. We can see that for our algorithm, the target arm pulls can increase linearly with $T$, when the value of $\Delta^{0}_{1K}$ is less than $\beta(1)$. But when the value of $\Delta^{0}_{1K}$ is greater than $\beta(1)$, our algorithm fails to achieve the goal. We should note that this is a result of detection against reward poisoning attacks, rather than a problem with the design of our algorithm. These results are exactly predicted by Lemma \ref{lemma: epsilon unattackable} and \ref{lemma: UCB1 unattackable}, and Theorem \ref{theorem: UCB1 attack} and \ref{theorem: epsilon attack}. 

\begin{figure*}[htbp]
    \centering
    \begin{tabular}{c c}
     \includegraphics[width=5cm]{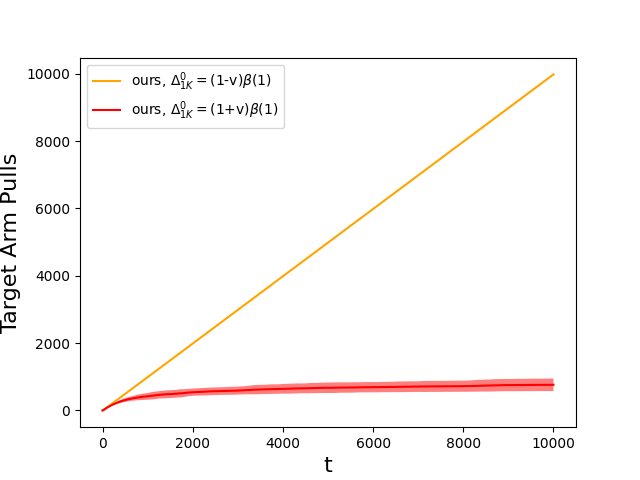} &
     \includegraphics[width=5cm]{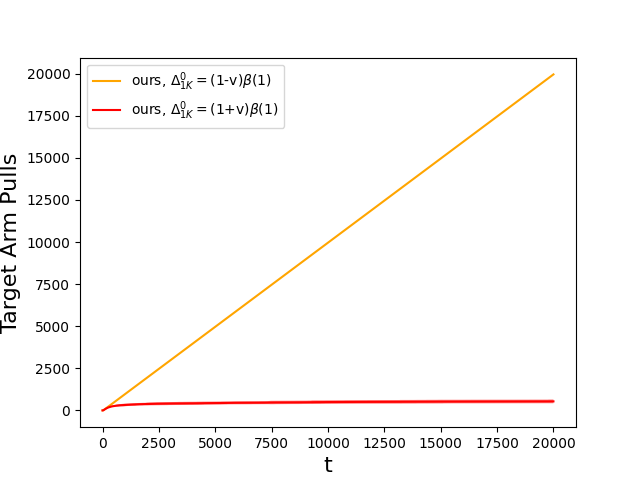} \\
     (a) $N = 10, T = 10000$. &
     (b) $N = 30, T = 20000$.
    % \vspace{-2mm}
    \end{tabular}
    \caption{Target arm pulls under different conditions when UCB1 is the victim algorithm.}
    
    \label{fig:armpull givenreward-UCB1}
\end{figure*}

\begin{figure*}[htbp]
    \centering
    \begin{tabular}{c c}
     \includegraphics[width=5cm]{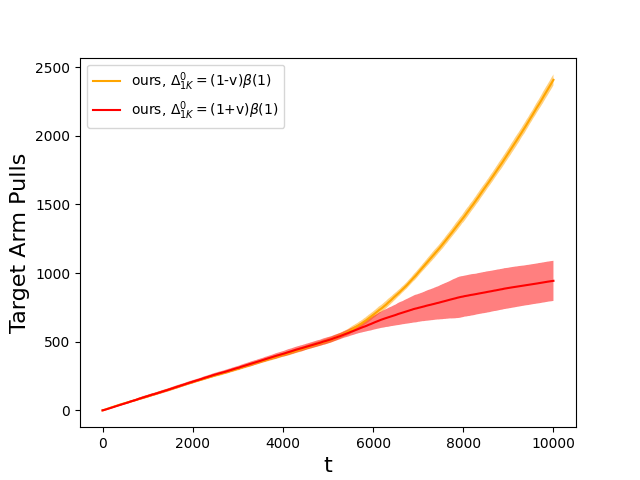} &
     \includegraphics[width=5cm]{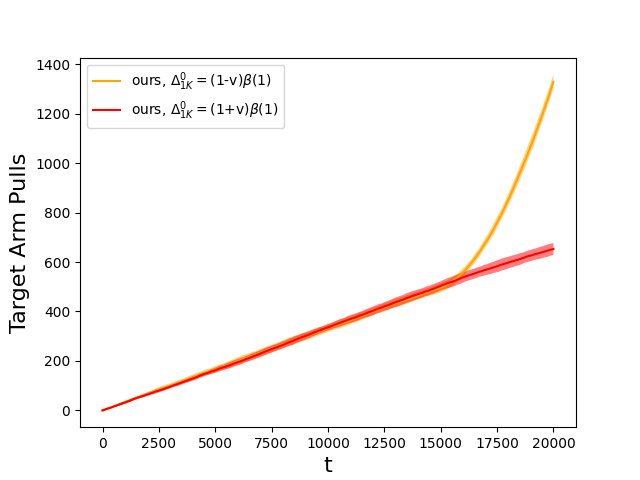} \\
     (a) $N = 10, T = 10000$. &
     (b) $N = 30, T = 20000$.
    % \vspace{-2mm}
    \end{tabular}
    \caption{Target arm pulls under different conditions when $\epsilon$-greedy is the victim algorithm.}\label{fig:armpull givenreward-egreedy}
\end{figure*}

It is also interesting to investigate the detection time under the two attack methods. We can see that the attack on the UCB1 algorithm will be quickly and successfully detected after all arms have been played once. In addition, we can see that for the $\epsilon$-greedy algorithm, the distribution of attack detection time is a little more dispersed, but mainly concentrated at the moment after all arms have been played once. The reason why successful detection takes place so early in the baseline is that it pulls down the empirical average reward of the non-target arm by a large amount, which makes this attack very easy to be detected. For $\epsilon$-greedy, because instead of choosing the best-performing arm, the choice of arm is random when $T$ is smaller than $500N$, the magnitude of the attack is relatively smaller, and thus the time for the attack to be detected will be relatively more scattered.

\begin{figure*}[htbp]
    \centering
    \begin{tabular}{c c c c}
    \begin{small}(N, T) = (10,10000) \end{small}&
     \includegraphics[width=3cm]{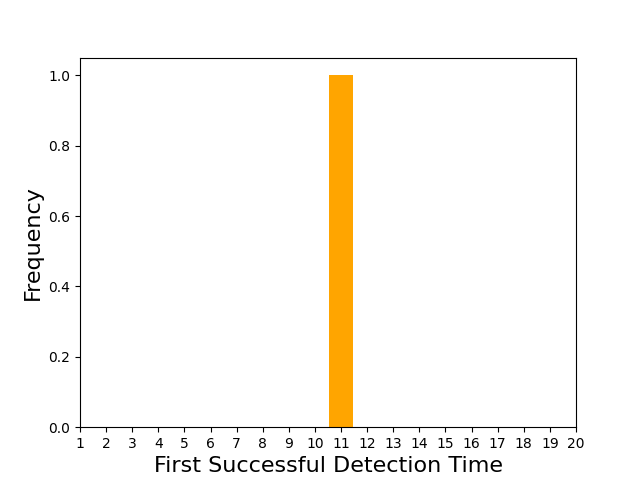} &
     \includegraphics[width=3cm]{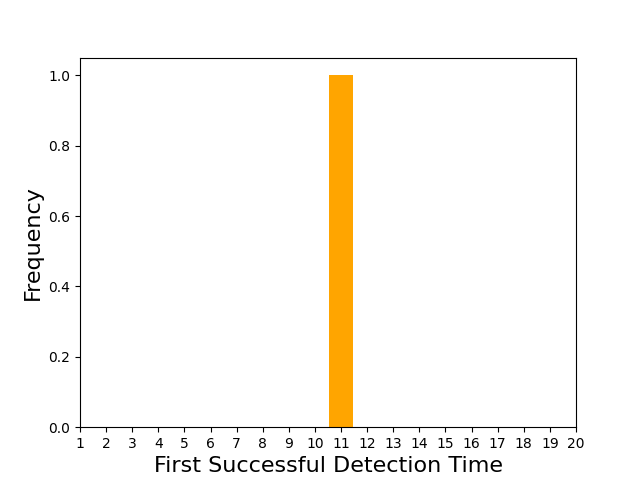} &
     \includegraphics[width=3cm]{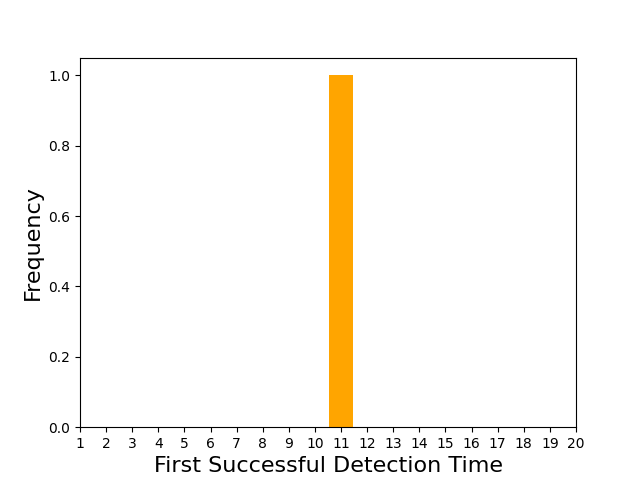}
     \\

     \begin{small}(N, T) = (30,20000) \end{small}&
     \includegraphics[width=3cm]{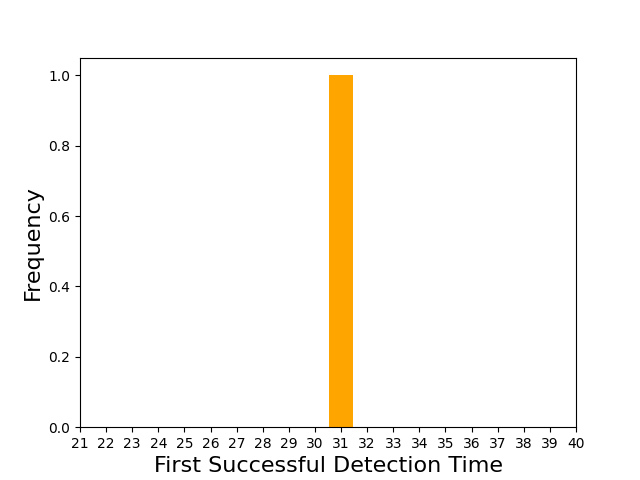} &
     \includegraphics[width=3cm]{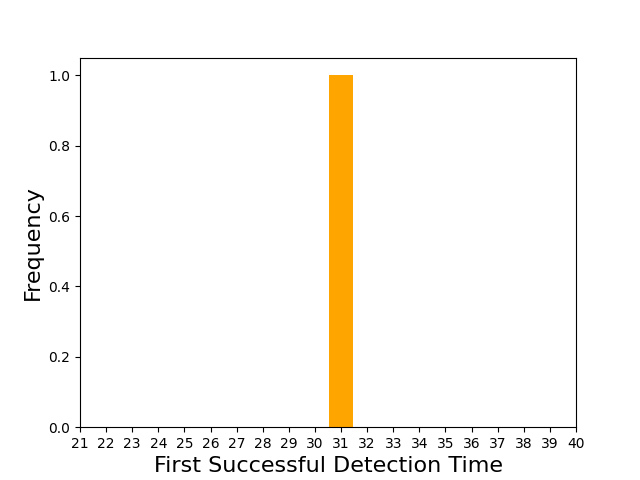} &
     \includegraphics[width=3cm]{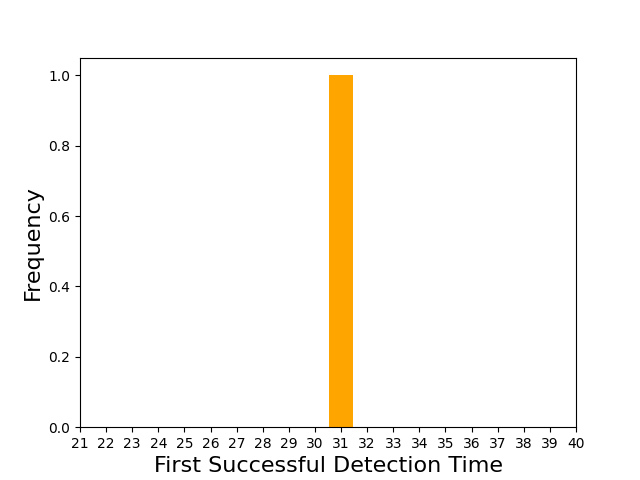}
     \\
     &
     (a) $\Delta_{1K}=\beta(1)/2$ &
     (b) $\Delta_{1K}=\beta(1)$ & 
     (c) $\Delta_{1K}=2\beta(1)$
    % \vspace{-2mm}
    \end{tabular}
    \caption{The distribution of the first successful detection time under different $\Delta_{1K}$  and victim algorithm UCB1. }
    \label{fig:detect-time-ucb1}
\end{figure*}

\begin{figure*}[htbp]
     \centering
     \begin{tabular}{c c c c}
    \begin{small}(N, T) = (10,10000) \end{small}&
     \includegraphics[width=3cm]{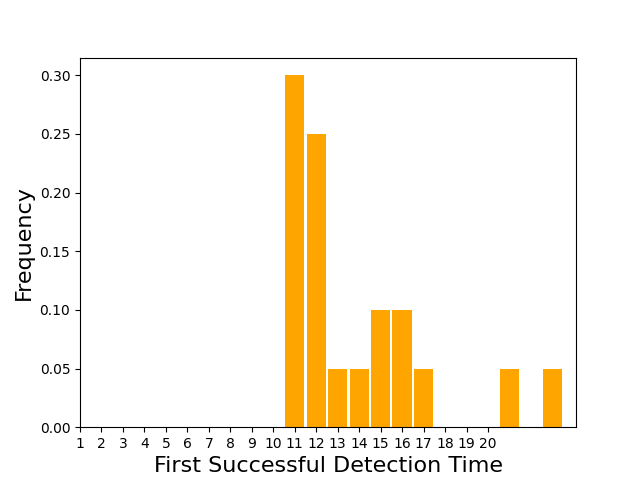} &
     \includegraphics[width=3cm]{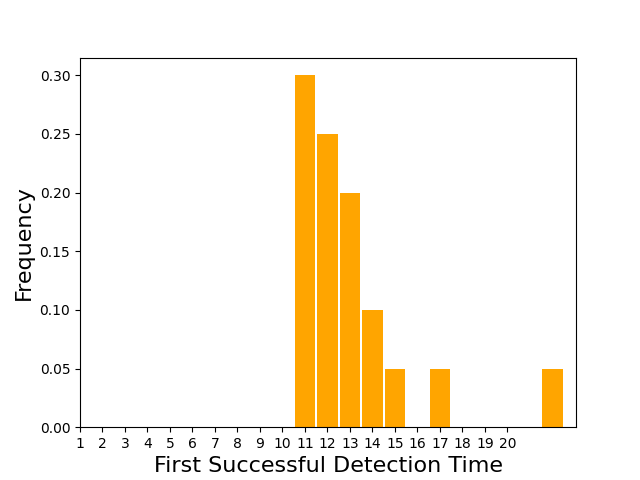} &
     \includegraphics[width=3cm]{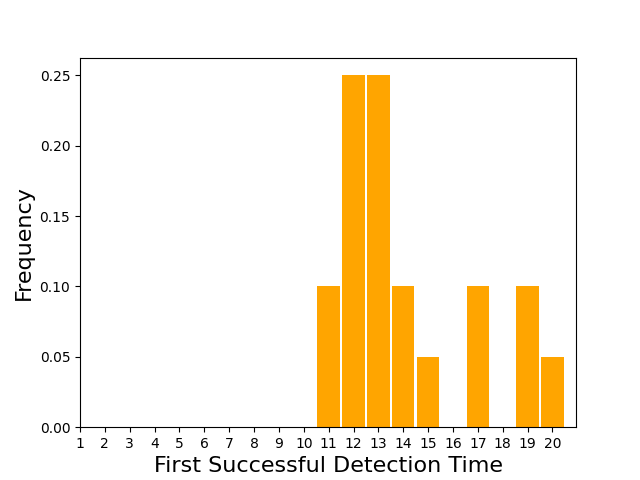} \\
    
     \begin{small}(N, T) = (30,20000) \end{small}&
    
    \includegraphics[width=3cm]{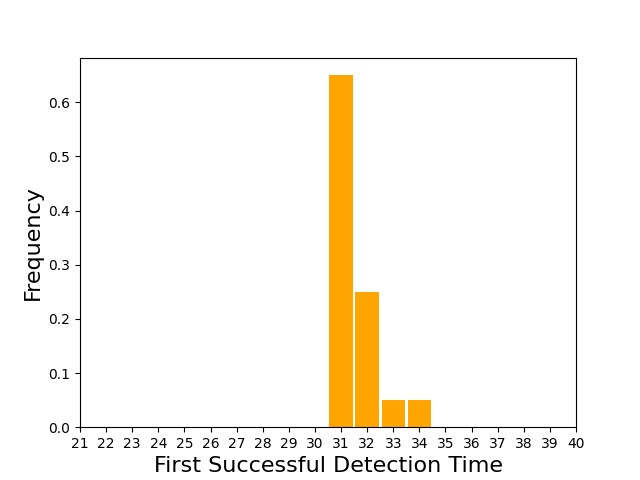} &
     \includegraphics[width=3cm]{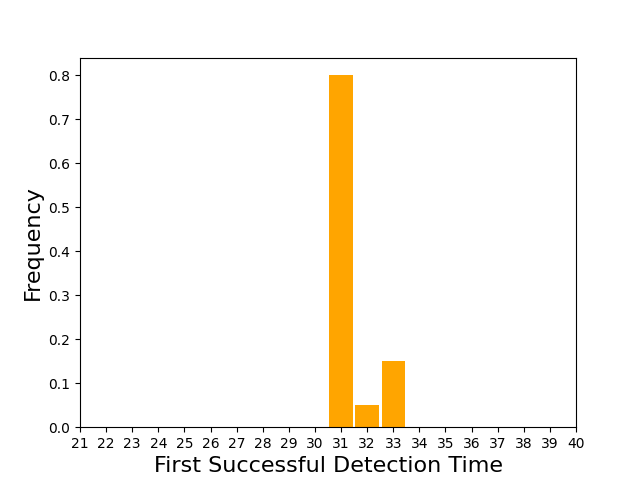} &
     \includegraphics[width=3cm]{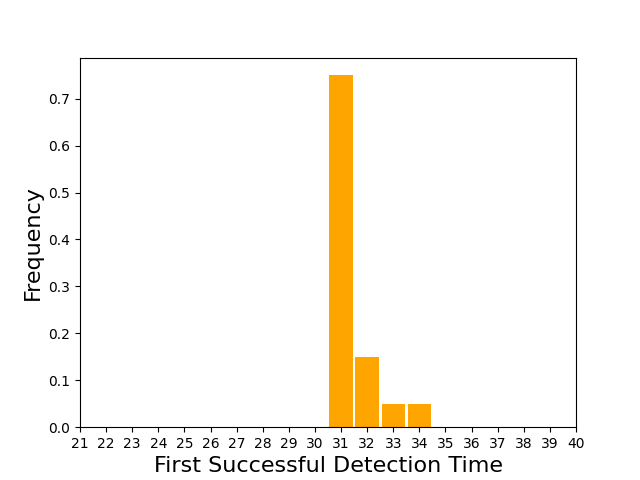} \\
     &
     (a) $\Delta_{1K}=\beta(1)/2$ & 
     (b) $\Delta_{1K}=\beta(1)$ & 
     (c) $\Delta_{1K}=2\beta(1)$
    % \vspace{-2mm}
    \end{tabular}
    \caption{The distribution of the first successful detection time under different $\Delta_{1K}$ and victim algorithm $\epsilon$-greedy.}
    \label{fig:detect-time-eps}
    \vspace{-3mm}
\end{figure*}

%%%%%%%%%%%%%%%%%%%%%%%%%%%%%%%%%%%%%%%%%%%%%%%%%%%%%%%%%%%%%%%%%%%%%%%%%%%%%%%
%%%%%%%%%%%%%%%%%%%%%%%%%%%%%%%%%%%%%%%%%%%%%%%%%%%%%%%%%%%%%%%%%%%%%%%%%%%%%%%